\documentclass[a4paper,11pt]{article}
\usepackage[T1]{fontenc}
\usepackage{url}
\usepackage[utf8]{inputenc}
\usepackage{authblk}
\usepackage{amsmath,graphicx,multicol}
\usepackage{epstopdf}
\usepackage{amsfonts}
\usepackage{bbm}
\usepackage{cite}
\usepackage{amssymb}
\usepackage{algorithm}
\usepackage{algorithmic}
\usepackage{tabularx}
\usepackage[utf8]{inputenc}
\usepackage[english]{babel}
\usepackage{amsthm}
\usepackage{mathtools}
\usepackage{algorithm}
\usepackage{algorithmic}
\usepackage{float}
\usepackage[margin=1in]{geometry}
\usepackage{multicol}
\usepackage{underscore}
\usepackage[english]{babel}
\usepackage{subcaption}
\usepackage{booktabs}
\usepackage[font=small,labelsep=period]{caption}
\usepackage{color}
\theoremstyle{definition}

\newtheorem{theorem}{Theorem}[]
\newtheorem{corollary}{Corollary}[theorem]
\newtheorem{lemma}[theorem]{Lemma}
\let\oldemptyset\emptyset
\let\emptyset\varnothing

\DeclarePairedDelimiter\floor{\lfloor}{\rfloor}
\DeclareMathOperator{\E}{\mathbb{E}}
\def\SNR{{\mathrm{SNR}}}
\def\q{{\mathbf q}}
\def\t{{\mathbf t}}
\def\x{{\mathbf x}}
\def\y{{\mathbf y}}
\def\e{{\boldsymbol{\nu}}}
\def\B{{\mathbf B}}

\def\C{{\mathbf C}}

\def\P{{\mathbf P}}

\def\R{{\mathbb{R}}}

\def\I{{\mathbf I}}

\def\r{{\mathbf r}}
\def\F{{\boldsymbol{\Phi}}}
\def\Si{{\boldsymbol{\Psi}}}

\def\x{{\mathbf x}}
\def\w{{\mathbf w}}

\def\r{{\mathbf r}}
\def\A{{\mathbf A}}
\def\B{{\mathbf B}}
\def\P{{\mathbf P}}

\def\R{{\mathbb{R}}}

\def\I{{\mathbf I}}
\def\a{{\mathbf a}}
\def\b{{\mathbf b}}
\def\r{{\mathbf r}}
\def\u{{\mathbf u}}
\def\c{{\mathbf c}}
\def\v{{\boldsymbol{\nu}}}

\newcommand{\ts}{\textsuperscript}
\newcommand\norm[1]{\left\lVert#1\right\rVert}
\newcommand{\argmax}{\arg\!\max}
\newcommand{\argmin}{\arg\!\min}

\let\oldref\ref
\renewcommand{\ref}[1]{(\oldref{#1})}
\newcommand{\RNum}[1]{\uppercase\expandafter{\romannumeral #1\relax}}

\makeatletter
\renewcommand{\fnum@figure}{Fig. \thefigure}
\makeatother

\definecolor{myred}{rgb}{0.95, 0.01, 0.01}

\title{\vspace{-1.7cm}\textbf{Accelerated Orthogonal Least-Squares for Large-Scale Sparse Reconstruction}}
\date{}
\author[]{Abolfazl Hashemi}
\author[]{Haris Vikalo}
\affil[]{Department of Electrical and Computer Engineering,\\University of Texas at Austin, USA}
\begin{document}
\maketitle
\vspace{-1cm}
\begin{abstract}
\noindent{
We study the problem of inferring a sparse vector from random linear combinations of its components. 
We propose the Accelerated Orthogonal Least-Squares (AOLS) algorithm that improves performance 
of the well-known Orthogonal Least-Squares (OLS) algorithm while requiring significantly lower 
computational costs. While OLS greedily selects columns of the coefficient matrix that correspond to
non-zero components of the sparse vector, AOLS employs a novel computationally efficient procedure 
that speeds up the search by anticipating future selections via choosing $L$ columns in each step,
where $L$ is an adjustable hyper-parameter. We analyze the performance of AOLS and establish 
lower bounds on the probability of exact recovery for both noiseless and noisy random linear 
measurements. In the noiseless scenario, it is shown that when the coefficients are samples from a 
Gaussian distribution, AOLS with high probability recovers a $k$-sparse $m$-dimensional sparse 
vector using ${\cal O}(k\log \frac{m}{k+L-1})$ measurements. Similar result is established for the 
bounded-noise scenario where an additional condition on the smallest nonzero element of the 
unknown vector is required. The asymptotic sampling complexity of AOLS is lower than the 
asymptotic sampling complexity of the existing sparse reconstruction algorithms. In simulations, 
AOLS is compared to state-of-the-art sparse recovery techniques and shown to provide better 
performance in terms of accuracy, running time, or both. Finally, we consider an application of 
AOLS to clustering high-dimensional data lying on the union of low-dimensional subspaces and 
demonstrate its superiority over existing methods.
}
\end{abstract}
\section{Introduction}\label{sec:intro}
The task of estimating sparse signal from a few linear combinations of its components is readily cast as the 
problem of finding a sparse solution to an underdetermined system of linear equations. Sparse recovery is 
encountered in many practical scenarios, including compressed sensing \cite{donoho2006compressed}, 
subspace clustering \cite{elhamifar2009sparse,you2015sparse}, sparse channel estimation 
\cite{carbonelli2007sparse,barik2014sparsity}, compressive DNA microarrays \cite{parvaresh2008recovering}, 
and a number of other applications in signal processing and machine learning 
\cite{lustig2007sparse,elad2010role,tipping2001sparse}. Consider the linear measurement model
\begin{equation} \label{eq:1}
{\y=\A\x+\e},
\end{equation}
where $\y\in\R^{n}$ denotes the vector of observations, $\A\in\R^{n \times m}$ is the coefficient matrix 
(i.e., a collection of features) assumed to be full rank (generally, $n < m$), $\e\in\R^{n}$ is the additive 
measurement noise vector, and $\x\in\R^{m}$ is an unknown vector assumed to have at most $k$ 
non-zero components (i.e., $k$ is the sparsity level of $\x$). Finding a sparse approximation to $\x$ 
leads to a cardinality-constrained least-squares problem
\begin{equation}  \label{eq:2}
\begin{aligned}
& \underset{\x}{\text{minimize}}
&  \norm{\y-\A\x}^{2}_{2}
& &\text{subject to}
& & \norm{\x}_{0} \leq k,
\end{aligned}
\end{equation}
known to be NP-hard; here $\|\cdot\|_0$ denotes the $\ell_0$-norm, i.e., the number of non-zero 
components of its argument. The high cost of finding the exact solution to \ref{eq:2} motivated 
development of a number of heuristics that can generally be grouped in the following categories:

1) \textit{Convex relaxation schemes.} These methods perform computationally efficient search 
for a sparse solution by replacing the non-convex $\ell_0$-constrained optimization by a 
sparsity-promoting $\ell_1$-norm optimization. It was shown in \cite{candes2005decoding} that 
such a formulation enables exact recovery of a sufficiently sparse signal from noise-free 
measurements under certain conditions on $\A$ and with $\mathcal{O}(k\log \frac{m}{k})$ 
measurements. However, while the convexity of $\ell_{1}$-norm enables algorithmically 
straightforward sparse vector recovery by means of, e.g., iterative shrinkage-thresholding 
\cite{beck2009fast} or alternating direction method of multipliers \cite{boyd2011distributed}, 
the complexity of such methods is often prohibitive in settings where one deals with 
high-dimensional signals.

2) \textit{Greedy schemes.} These heuristics attempt to satisfy the cardinality constraint directly 
by successively identifying columns of the coefficient matrix which correspond to non-zero 
components of the unknown vector. Among the greedy methods for sparse vector reconstruction, 
the orthogonal matching pursuit (OMP) \cite{pati1993orthogonal} and Orthogonal Least-Squares 
(OLS) \cite{chen1989orthogonal,natarajan1995sparse} have attracted particular attention in 
recent years. Intuitively appealing due to its simple geometric interpretation, OMP is characterized 
by high speed and competitive performance. In each iteration, OMP selects a column of the 
coefficient matrix $\A$ having the highest correlation with the so-called residual vector and adds 
it to the set of active columns; then the projection of the observation vector $\y$ onto the space 
spanned by the columns in the active set is used to form a residual vector needed for the next 
iteration of the algorithm. Numerous modifications of OMP with enhanced performance have been 
proposed in literature. For instance, instead of choosing a single column in each iteration of OMP, 
StOMP \cite{donoho2012sparse} selects and explores all columns having correlation with a 
residual vector that is greater than a pre-determined threshold. GOMP\cite{wang2012generalized} 
employs  the similar idea, but instead of thresholding, a fixed number of columns is selected per 
iteration. CoSaMP algorithm \cite{needell2009cosamp} identifies columns with largest proximity to 
the residual vector, uses them to find a least-squares approximation of the unknown signal, and 
retains only significantly large entries in the resulting approximation. 
Additionally, necessary and sufficient conditions for exact reconstruction of sparse signals using 
OMP have been established. Examples of such results include analysis under Restricted Isometry 
Property (RIP) \cite{zhang2011sparse,davenport2010analysis,mo2012remark}, and recovery 
conditions based on Mutual Incoherence Property (MIP) and Exact Recovery Condition (ERC) 
\cite{tropp2004greed,cai2011orthogonal,zhang2009consistency}. For the case of random 
measurements, performance of OMP was analyzed in \cite{tropp2007signal,fletcher2012orthogonal}. 
Tropp et al. in \cite{tropp2007signal} showed that in the noise-free scenario, 
${\cal O}\left(k\log m\right)$ measurements is adequate to recover $k$-sparse $m$-dimensional 
signals with high probability. In \cite{rangan2009orthogonal}, this result was extended to the 
asymptotic setting of noisy measurements in high signal-to-noise ratio (SNR) under the assumption 
that the entries of $\A$ are i.i.d Gaussian and that the length of the unknown vector approaches 
infinity. Recently, the asymptotic sampling complexity of OMP and GOMP is improved to 
${\cal O}(k\log \frac{m}{k})$ in \cite{foucart2012stability} and \cite{wang2016recovery}, 
respectively.

Recently, performance of OLS  was analyzed in the sparse signal recovery settings with deterministic coefficient matrices. In \cite{soussen2013joint}, OLS was analyzed in the noise-free scenario under Exact Recovery Condition (ERC), first introduced in\cite{tropp2004greed}. Herzet et al. \cite{herzet2016relaxed} provided coherence-based conditions for sparse recovery of signals via OLS when the nonzero components of ${\bf x}$ obey certain decay conditions. In \cite{herzet2013exact}, sufficient conditions for exact recovery are stated when a subset of true indices is available. In \cite{wang2017recovery} an extension of OLS that employs the idea of \cite{donoho2012sparse,wang2012generalized} and identifies multiple indices in each iteration is proposed and its performance is analyzed under RIP. However, all the existing analysis and performance guarantees for OLS pertain to non-random measurements and cannot directly be applied to random coefficient matrices. For instance, the main results in the notable work \cite{foucart2012stability} relies on the assumption of having dictionaries with $\ell_2$-norm normalized columns while this obviously does not hold in the scenarios where the coefficient matrix is composed of entries that are drawn from a 
Gaussian distribution. 

3) \textit{Branch-and-bound schemes.} Recently, greedy search heuristics that rely on OMP and 
OLS to traverse a search tree along paths that represent promising candidates for the support of 
$\x$ have been proposed. For instance, \cite{karahanoglu2012orthogonal,kwon2014multipath} 
exploit the selection criterion of OMP to construct the search graph while 
\cite{maymon2015viterbi,hashemi2017sparse} rely on OLS to efficiently traverses the search tree. 
Although these methods empirically improve the performance of greedy algorithms, they are 
characterized by exponential computational complexity in at least one parameter and hence are 
prohibitive in applications dealing with high-dimensional signals.
\subsection{Contributions}\label{sec:intro_contribution}
Motivated by the need for fast and accurate sparse recovery in large-scale setting, in this paper 
we propose a novel algorithm that efficiently exploits recursive relation between components of 
the optimal solution to the original $\ell_0$-constrained least-squares problem \ref{eq:2}. The 
proposed algorithm, referred to as Accelerated Orthogonal Least-Squares (AOLS), similar to 
GOMP \cite{wang2012generalized} and MOLS \cite{wang2017recovery} exploits the observation 
that columns having strong correlation with the current residual are likely to have strong correlation 
with residuals in subsequent iterations; this justifies selection of multiple columns in each iteration 
and formulation of an overdetermined system of linear equation having solution that is generally 
more accurate than the one found by OLS or OMP. However, compared to MOLS, our proposed algorithm is orders of magnitude faster and thus more suitable for high-dimensional data 
applications.

We theoretically analyze the performance of the proposed AOLS algorithm and, by doing so, establish conditions for the 
exact recovery of the sparse vector $\x$ from measurements $\y$ in \ref{eq:1} when the entries 
of the coefficient matrix $\A$ are drawn at random from a Gaussian distribution 
{\color{black}{-- the first such result under these assumptions for an OLS-based algorithm}}. We first present conditions which 
ensure that, in the noise-free scenario, AOLS with high probability recovers the support of $\x$ 
in $k$ iterations (recall that $k$ denotes the number of non-zero entries of ${\bf x}$). Adopting 
the framework in \cite{tropp2007signal}, we further find a lower bound on the probability of 
performing exact sparse recovery in $k$ iterations and demonstrate that with 
${\cal O}\left(k\log \frac{m}{k+L-1}\right)$ measurements AOLS succeeds with probability arbitrarily 
close to one. Moreover, we extend our analysis to the case of noisy measurements and show that 
similar guarantees hold if the nonzero element of $\x$ with the smallest magnitude satisfies certain 
condition. This condition implies that to ensure exact support recovery via AOLS in the presence of 
additive $\ell_2$-bounded noise, $\SNR$ should scale linearly with sparsity level $k$. Our 
procedure for determining requirements that need to hold for AOLS to perform exact reconstruction 
follows the analysis of OMP in \cite{tropp2007signal,rangan2009orthogonal,fletcher2012orthogonal}, 
although with two major differences. First, the variant of OMP analyzed in 
\cite{tropp2007signal,rangan2009orthogonal,fletcher2012orthogonal} implicitly assumes that the 
columns of $\A$ are $\ell_2$-normalized which clearly does not hold if the entries of $\A$ are drawn 
from a Gaussian distribution. Second, the analysis in \cite{tropp2007signal} is for noiseless 
measurements while \cite{rangan2009orthogonal,fletcher2012orthogonal} essentially assume that 
$\SNR$ is infinite as $k \rightarrow \infty$. To the contrary, our analysis makes neither of those two 
restrictive assumptions. Moreover, we show that if $m$ is sufficiently greater than $k$, the proposed 
AOLS algorithm requires ${\cal O}\left(k\log \frac{m}{k+L-1}\right)$ random measurements to perform 
exact recovery in both noiseless and bounded noise scenarios; this is fewer than 
${\cal O}\left(k\log (m-k)\right)$ that was found in \cite{rangan2009orthogonal,fletcher2012orthogonal} 
to be the asymptotic sampling complexity for OMP, and ${\cal O}\left(k\log  \frac{m}{k}\right)$ that was 
found for MOLS, GOMP, and BP in \cite{wang2017recovery,wang2016recovery,candes2006robust}.  
Additionally, our analysis framework is recognizably different from that of \cite{foucart2012stability} 
for OLS. First, in \cite{foucart2012stability} it is assumed that $\A$ has $\ell_2$-normalized columns, 
and hence the analysis in \cite{foucart2012stability} does not apply to the case of Gaussian matrices, 
the scenario addressed in this paper for our proposed algorithm. Further, the main result of 
\cite{foucart2012stability} (see Theorem 3 in  \cite{foucart2012stability}) states that OLS exactly 
recovers a $k$-sparse vector in at most $6k$ iterations if ${\cal O}\left(k\log  \frac{m}{k}\right)$ 
measurements are available. Hence, the OLS results of \cite{foucart2012stability} are not as strong
as the AOLS results we establish in the current paper. Our extensive empirical studies verify the 
theoretical findings and demonstrate that AOLS is more accurate and faster than the competing
state-of-the-art schemes.

To further demonstrate efficacy of the proposed techniques, we consider the sparse subspace 
clustering (SSC) problem that is often encountered in machine learning and computer vision 
applications. The goal of SSC is to partition data points drawn from a union of low-dimensional 
subspaces. We propose a SSC scheme that relies on our AOLS algorithm and empirically show 
significant improvements in accuracy compared to state-of-the-art methods in \cite{you2015sparse,dyer2013greedy,elhamifar2009sparse,elhamifar2013sparse}.

\subsection{Organization}
The remainder of the paper is organized as follows. In Section \oldref{sec:not}, we specify the 
notation and overview the classic OLS algorithm. In Section \oldref{sec:accalgo}, we describe the proposed AOLS algorithm. Section \oldref{sec:Gua} presents analysis of the performance of AOLS 
for sparse recovery from random measurements. Section \oldref{sec:sim} presents experiments 
that empirically verify our theoretical results on sampling requirements of AOLS and benchmark its performance. Finally, concluding remarks are provided in Section \oldref{sec:conc}. Matlab implementation of AOLS is freely available for download from \url{https://github.com/realabolfazl/AOLS/}.
\section{Preliminaries}\label{sec:not}
\subsection{Notation}
We briefly summarize notation used in the paper. Bold capital letters refer to matrices and bold lowercase letters represent 
vectors. Matrix $\A \in \R^{n\times m}$ is assumed to have full rank; $\A_{ij}$ denotes the $(i,j)$ entry of
$\A$, $\a_j$ is the $j\ts{th}$ column of $\A$, and $\A_k \in \R^{n\times k}$ is one of the ${m}\choose{k}$ submatrices of $\A$ (here we assume $k<n<m$). 
${\cal L}_k$ denotes the subspace spanned by the columns of $\A_k$. $\P_k^\bot=\I-\A_k \A_k^\dagger$ is the projection operator 
onto the orthogonal complement of ${\cal L}_k$ where $\A_K^\dagger=\left(\A_k^{\top}\A_k\right)^{-1}\A_k^{\top}$ 
denotes the 
Moore-Penrose pseudo-inverse of $\A_k$ and $\I \in \R^{n\times n}$ is the identity matrix. 
${\cal I}=\{1,\dots,m\}$ is the set of column indices, 
${\cal S}_{true}$ is the set of indices of nonzero elements of $\x$, and ${\cal S}_i$ is the set of selected indices at the end of the $i\ts{th}$ iteration of OLS. 
For a non-scalar object such as matrix $\A$, $\A \sim {\cal N}\left(0,\frac{1}{n}\right)$ implies that
the entries of $\A$ are drawn independently from a zero-mean Gaussian distribution with variance $\frac{1}{n}$. Further, for any vector $\c \in \R^{m}$ define the ordering operator $\mathcal{P}:  \R^{m} \rightarrow  \R^{m}$ as $\mathcal{P}(\c) = [\c_{o_1},\dots,\c_{o_m}]^\top$ such that $|\c_{o_1}|\leq \dots \leq| \c_{o_m}|$. Finally, $\mathbf{1}$ denotes the vector of all ones, and $\mathcal{U}(0,q)$ represents the uniform distribution on $[0,q]$.

\subsection{The OLS Algorithm}\label{sec:not_ols}
The OLS algorithm sequentially projects columns of $\A$ onto a residual vector and selects the one resulting in the smallest 
residual norm. Specifically, in the $i\ts{th}$ iteration OLS chooses a new column index $j_s$ according to
\begin{equation}\label{eq:ools}
{j}_{s}=\argmin_{j \in {\cal I}\backslash {\cal S}_{i-1}}{\norm{\P_{{\cal S}_{i-1}\cup\{j\}}^\bot\y}_2}.
\end{equation}
This procedure is computationally more expensive than OMP since in addition to solving a least-squares problem to update the 
residual vector, orthogonal projections of the columns of $\A$  need to be found in each step of OLS. Note that the performances 
of OLS and OMP are identical when the columns of $\A$ are orthogonal.\footnote{Orthogonality of the columns of $\A$ implies 
that the objective function in \ref{eq:2} is modular; in this case and noiseless setting, both methods are optimal.} It is
worthwhile pointing out the difference between OMP and OLS. In each iteration of OMP, an element most correlated with the 
current residual is chosen. OLS, on the other hand, selects a column least expressible by previously selected columns which, 
in turn, minimizes the approximation error. 

It can be shown, see, e.g., \cite{rebollo2002optimized,soussen2013joint,hashemi2016sparse}, that the index selection criterion \ref{eq:ools} can alternatively be expressed as 
\begin{equation}\label{eq:nols}
j_s=\argmax_{j \in {\cal I}\backslash {\cal S}_{i-1}}{\left|\r_{i-1}^{\top}\frac{{\bf P}_{i-1}^{\bot}{\bf a}_{j}}{\norm{{\bf P}_{i-1}^{\bot}{\bf a}_{j}}_2}\right|},
\end{equation}
where $\r_{i-1}$ denotes the residual vector in the $i\ts{th}$ iteration. Moreover, projection matrix needed for the subsequent 
iteration is related to the current projection matrix according to
\begin{equation}
{\P}_{i+1}^{\bot}={\P}_i^{\bot}-\frac{{\P}_i^{\bot}\a_{j_s}\a_{j_s}^{\top} {\P}_i^{\bot}}{\norm{{\P}_i^{\bot}\a_{j_s}}_2^2}.
\label{eq:perp}
\end{equation}
It should be noted that $\r_{i-1}$ in \ref{eq:nols} can be replaced by $\y$ because of the idempotent property of the projection 
matrix,
\begin{equation}
{\P}_i^{\bot}={{\P}_i^{\bot}}^{\top}={\P_i^{\bot}}^2. 
\end{equation}
This substitution reduces complexity of OLS although, when sparsity level $k$ is unknown, the norm of $\r_{i}$ still needs to be 
computed since it is typically used when evaluating a stopping criterion. OLS is formalized as Algorithm 1. 
\begin{algorithm}[t]
\vspace{0.1cm}
\begin{tabularx}{\textwidth}{l>{$}c<{$}X}
\textbf{Input:}  \hspace{0.58cm} $\y$, $\A$, sparsity level $k$ \vspace{0.1cm}\\
\textbf{Output:} \hspace{0.28cm} recovered support ${\cal S}_k$, estimated signal $\hat{\x}_{k}$\vspace{0.1cm}\\
\textbf{Initialize:} \hspace{0.02cm} ${\cal S}_0=\oldemptyset$
, ${\bf P}_0^{\bot}={\bf I}$
\end{tabularx}
\begin{algorithmic}
\FOR  {\hspace{0.40cm}$i=1$ to $k$ \hspace{0.40cm}}\vspace{0.1cm}
\STATE 1. $j_s=\argmax_{j \in {\cal I}\backslash {\cal S}_{i-1}}{\left|\y^{\top}\frac{{\bf P}_{i-1}^{\bot}{\bf a}_{j}}{\norm{{\bf P}_{i-1}^{\bot}{\bf a}_{j}}_2}\right|}$ \vspace{0.1cm}\\
2. ${\cal S}_{i}={\cal S}_{i-1}\cup\{j_{s}\}$ 
\vspace{0.1cm}\vspace{0.1cm}\\
3. ${\P}_{i+1}^{\bot}={\P}_i^{\bot}-\frac{{\P}_i^{\bot}\a_{j_s}\a_{j_s}^{\top} {\P}_i^{\bot}}{\norm{{\P}_i^{\bot}\a_{j_s}}_2^2}$
\ENDFOR \\\vspace{0.1cm}
\STATE 4. $\hat{\x}_{k}=\A_{{\cal S}_{k}}^{\dagger}\y$\\ 
\end{algorithmic}
\caption{Orthogonal Least-Squares (OLS)}
\label{algo:ols}
\end{algorithm}
\section{A Novel Accelerated Scheme for Sparse Recovery}\label{sec:accalgo}
In this section we describe the AOLS algorithm in detail. The complexity of the OLS and its existing variants such as MOLS \cite{wang2017recovery} is dominated by the so-called identification and update steps, formalized as steps 1 and 3 of Algorithm 1 in Section \oldref{sec:not}, respectively; in these steps, the algorithm evaluates projections ${\bf P}_{i-1}^{\bot}{\bf a}_{j}$ of not-yet-selected columns onto the space spanned by the selected ones and then computes the projection matrix ${\bf P}_{i}$ needed for the next iteration. This becomes practically infeasible in applications that involve dealing with high-dimensional data, including sparse subspace clustering. To this end, in Theorem \oldref{thm:3} below, we establish a set of recursions which significantly reduce the complexity of the identification and update steps without sacrificing the performance. AOLS then relies on these efficient recursions to identify the indices corresponding to nonzero entries of $\x$ with a significantly lower computational costs with respect to OLS and MOLS. This is further verified in our simulation studies.
\begin{theorem}\label{thm:3}
\textit{Let $\r_{i}$ denote the residual vector in the $i^{th}$ iteration of OLS with $\r_0=\y$. The identification 
step (i.e., step 1 in Algorithm~1) in the $(i+1)\ts{st}$ iteration of OLS can be rephrased as}
\begin{equation}\label{eq:o4}
j_{s}=\argmax_{j \in {\cal I}\backslash {\cal S}_i }\norm{{\bf q}_j}_2,
\end{equation}
\textit{where}
\begin{equation}\label{eq:o5}
{\bf q}_j=\frac{\a_j^\top\r_{i}}{\a_j^\top {\bf t}_j^{(i)}}{\bf t}_j^{(i)}, \hspace{0.5cm}
{\bf t}_j^{(i+1)}={\bf t}_j^{(i)}-\frac{{{\bf t}_j^{(i)}}^{\top}\u_i}{\norm{\u_i}_2^2}\u_i,
\end{equation}
\textit{where ${\bf t}_j^{(0)}= \a_j$ for all $j\in \mathcal{I}$. Furthermore, the residual vector $\r_{i+1}$ required for the next iteration is formed as}
\begin{equation}\label{eq:aols}
\u_{i+1}={\bf q}_{j_s}, \hspace{0.5cm}\r_{i+1}=\r_{i}-\u_{i+1}.
\end{equation}
\end{theorem}
\begin{proof}
Assume that column $\a_{j_s}$ is selected in the $(i+1)\ts{st}$ iteration of the algorithm. Define 
$\bar{\q}_{j}=\frac{{\P}_i^{\bot}\a_{j}}{\norm{{\P}_i^{\bot}\a_{j}}_2^2}\a_{j}^\top \r_i$, $\forall j \in {\cal I}\backslash {\cal S}_i $. 
Therefore,
\begin{equation} \label{eq:normqq}
\argmax_{j \in {\cal I}\backslash {\cal S}_i}{\|\bar{\q}_{j}\|_2}=\argmax_{j \in {\cal I}\backslash {\cal S}_i}{\frac{|\a_{j}^\top \r_i|}
{\norm{{\P}_i^{\bot}\a_{j}}_2}}.
\end{equation}
The idempotent property of $\P_i^\bot$ implies that the right hand side of \ref{eq:normqq} is equivalent to the selection 
rule of OLS in \ref{eq:nols}. Therefore, $j_s=\argmax_{j \in {\cal I}\backslash {\cal S}_i}{\|\bar{\q}_{j_s}\|_2}$. Let us 
post-multiply both sides of \ref{eq:perp} with the observation vector $\y$, leading to
\begin{equation}
{\P}_{i+1}^{\bot}\y={\P}_i^{\bot}\y-\frac{{\P}_i^{\bot}\a_{j_s}\a_{j_s}^\top {\P}_i^{\bot}}{\norm{{\P}_i^{\bot}\a_{j_s}}_2^2}\y.
\end{equation}
Recall that $\r_i={\P}_i^{\bot}\y$ to obtain
\begin{equation}
\begin{aligned}
\r_{i+1}=\r_i-\frac{{\P}_i^{\bot}\a_{j_s}}{\norm{{\P}_i^{\bot}\a_{j_s}}_2^2}\a_{j_s}^\top \r_i=\r_i-\bar{\q}_{j_s}.
\end{aligned}
\end{equation}
Comparing the above expression with \ref{eq:o5}, one needs to show $\q_{j_s}=\bar{\q}_{j_s}$ to complete the proof;
this is equivalent to demonstrating 
$\frac{\P_i^\bot\a_{j_s}}{\norm{\P_i^\bot\a_{j_s}}_2^2}=\frac{1}{\a_{j_s}^\top {\bf t}_j^{(i)}}{\bf t}_j^{(i)}$. 
Since $\A$ is full rank, the selected columns are linearly independent. 
Let $ \{\widetilde{\a}_l\}_{l=1}^i$ denote the collection of columns selected in the first $i$ iterations and let 
${\cal L}_i=\{\widetilde{\a}_1,\dots,\widetilde{\a}_i\}$ denote the subspace spanned by those columns. Consider 
the orthogonal projection of the selected column $\a_{j_s}$ onto ${\cal L}_i$, ${\P}_i^{\bot}\a_{j_s}$. Clearly,
${\P}_i^{\bot}\a_{j_s}=\a_{j_s}-{\P}_i\a_{j_s}$. Noting the idempotent property of ${\P}_i^\bot$ and the fact that 
$\|\a_{j_S}\|_2^2=\|{\P}_i^{\bot}\a_{j_s}\|_2^2+\|{\P}_i\a_{j_s}\|_2^2$, we obtain 
\begin{equation}
\frac{\P_i^\bot\a_{j_s}}{\norm{\P_i^\bot\a_{j_s}}_2^2}=\frac{\a_{j_s}-{\P}_i\a_{j_s}}{\a_{j_s}^\top\left(\a_{j_s}-{\P}_i\a_{j_s}\right)}.
\end{equation}
We need to demonstrate that ${\P}_i\a_{j_s}=\sum_{l=1}^{i}\frac{\a_{j_s}^{\top}\u_l}{\norm{\u_l}_2^2}\u_l$, i.e., 
$ \{\u_l\}_{l=1}^i$ is an orthogonal basis for ${\cal L}_i$. To this end, we employ an inductive argument. Consider 
$\u_1$ and $\u_2$ associated with the  $1\ts{st}$ and $2\ts{nd}$ iterations. Applying \ref{eq:o5} yields
\begin{equation}
\u_1=\frac{\widetilde{\a}_1^\top\r_0}{\|\widetilde{\a}_i\|_2^2}\widetilde{\a}_1,
\end{equation}
\begin{equation}
\u_2=\frac{\widetilde{\a}_2^\top(\r_0-\u_1)}{\widetilde{\a}_2^\top\left(\widetilde{\a}_2-\frac{\widetilde{\a}_2^\top\u_1}{\|\u_1\|
_2^2}\u_1\right)}\left(\widetilde{\a}_2-\frac{\widetilde{\a}_2^\top\u_1}{\|\u_1\|_2^2}\u_1\right).
\end{equation}
It is straightforward to see that
$\widetilde{\a}_1^\top\left(\widetilde{\a}_2-\frac{\widetilde{\a}_2^\top\u_1}{\|\u_1\|_2^2}\u_1\right)=0$; therefore,  
$\u_1^\top\u_2=0$. Now, a collection of orthogonal columns $\{\u_l\}_{l=1}^{i-1}$ forms a basis for ${\cal L}_{i-1}$. 
It follows from \ref{eq:o5} that
\begin{equation}
\u_i=\frac{\widetilde{\a}_i^\top(\r_{i-2}-\u_{i-1})}{\widetilde{\a}_i^\top\left(\widetilde{\a}_i-\sum_{l=1}^{i-1}
\frac{\widetilde{\a}_i^\top\u_l}{\|\u_1\|_2^2}\u_l\right)}\left(\widetilde{\a}_i-\sum_{l=1}^{i-1}
\frac{\widetilde{\a}_i^\top\u_l}{\|\u_1\|_2^2}\u_l\right).
\end{equation}
Consider $\u_l^\top\u_i$ for any $l\in \{1,\dots, i-1\}$. Since the collection $\{\u_l\}_{l=1}^{i-1}$ is orthogonal, 
$\u_l^\top\u_i$ is proportional to 
$\widetilde{\a}_l^\top\left(\widetilde{\a}_i-\frac{\widetilde{\a}_i^\top\u_l}{\|\u_1\|_2^2}\u_l\right)$, which is readily shown to be 
zero.
Consequently, $ \{\u_l\}_{l=1}^i$ is an orthogonal basis for ${\cal L}_i$ and the orthogonal projection of $\a_{j_s}$ is 
formed as the Euclidean projection of $\a_{j_s}$ onto each of the orthogonal vectors $\u_l$. Therefore, 
${\P}_i\a_{j_s}=\sum_{l=1}^{i}\frac{\a_{j_s}^{\top}\u_l}{\norm{\u_l}_2^2}\u_l$. Using a similar inductive argument one can show ${\bf t}_j^{(i+1)} = {\bf t}_j^{(i)} -\frac{{{\bf t}_j^{(i)}}^{\top}\u_i}{\norm{\u_i}_2^2}\u_i$
, which completes the proof.
\end{proof}
The geometric interpretation of the recursive equations established in Theorem \oldref{thm:3} is stated in 
Corollary \oldref{co:ols}. Intuitively, after orthogonalizing selected columns, a new column is identified and 
added it to the subset thus expanding the corresponding subspace.
\begin{corollary} \label{co:ols}
\textit{Let $ \{\widetilde{\a}_l\}_{l=1}^i$ denote the set of columns selected in the first $i$ iterations of the OLS algorithm  and let ${\cal L}=\{\widetilde{\a}_1,\dots,\widetilde{\a}_i\}$ be the subspace spanned by these columns. Then $ \{\u_l\}_{l=1}^i$ generated according to Theorem \oldref{thm:3} forms an orthogonal basis for ${\cal L}_i$.}
\end{corollary}
Selecting multiple indices per iteration was first proposed in \cite{donoho2012sparse,wang2012generalized} 
and shown to improve performance while reducing the number of OMP iterations. However, since selecting 
multiple indices increases computational cost of each iteration, relying on OMP/OLS identification criterion 
(as in, e.g., \cite{wang2017recovery}) does not necessarily reduce the complexity and may in 
fact be prohibitive in practice, as we will demonstrate in our simulation results. Motivated by this observation,
we rely on recursions derived in Theorem \oldref{thm:3} to develop a novel, computationally efficient variant of
OLS that we refer to as Accelerated OLS (AOLS) and formalize it as Algorithm 2. The proposed AOLS algorithm 
starts with $\mathcal{S}_0=\emptyset$ and, in each step, selects $1\leq L \leq \floor{\frac{n}{k}}$ columns of 
matrix $\A$ such that their normalized projections onto the orthogonal complement of the subspace spanned 
by the previously chosen columns have higher correlation with the residual vector than remaining non-selected 
columns. That is, in the $i\ts{th}$ iteration, AOLS identifies $L$ indices 
$\{{s_1},\dots,{s_L}\} \subset \mathcal{I}\backslash \mathcal{S}_{i-1}$ corresponding to the $L$ largest terms
$\|\q_j\|_2^2$. After such indices are identified, AOLS employs \ref{eq:aols} to repeatedly update the residual 
vector required for consecutive iterations. The procedure continues until a stopping criterion (e.g., a 
predetermined threshold on the norm of the residual vector) is met, or a preset maximum number of iterations 
is reached.

\textit{Remark:} As we will show in our simulation results, performance of AOLS is equivalent to that of the 
MOLS algorithm. However, AOLS is much faster and more suitable for real-world applications involving 
high-dimensional signals. In fact, it is straightforward to see that the worst case computational costs of 
Algorithm 1 (and also MOLS) and Algorithm 2 are ${\cal O}\left(mn^2k\right)$ and ${\cal O}\left(mnk\right)$, 
respectively; therefore, AOLS is significantly less complex than the conventional OLS and MOLS
algorithms.
\renewcommand\algorithmicdo{}
\begin{algorithm}[t]
\vspace{0.1cm}
\begin{tabularx}{\textwidth}{l>{$}c<{$}X}
\textbf{Input:}  \hspace{0.58cm} $\y$, $\A$, sparsity level $k$, threshold $\epsilon$, $1\leq L \leq \floor{\frac{n}{k}}$ \vspace{0.1cm}\\
\textbf{Output:} \hspace{0.28cm} recovered support ${\cal S}_k$, estimated signal $\hat{\x}_{k}$\vspace{0.1cm}\\
\textbf{Initialize:} \hspace{0.02cm}  $i = 0$, ${\cal S}_i=\oldemptyset$, $\r_i=0$, ${\bf t}_j^{(i)}=\a_j$, $\q_j = \frac{\a_j^\top\r_i}{\a_j^\top {\bf t}_j^{(i)}}{\bf t}_j^{(i)}$ for all $j \in \mathcal{I}$.
\end{tabularx}
\begin{algorithmic}
\WHILE  {  $\|\r_i\|_2\geq \epsilon$ and $i<k$ }\vspace{0.1cm}
\STATE 1. Select $\{j_{s_1},\dots,j_{s_L}\}$ corresponding to $L$ largest terms
$\norm{{\bf q}_j}_2$ 
\vspace{0.1cm}\\
2. $i \leftarrow i+1$
\vspace{0.1cm}\\
3. ${\cal S}_{i}={\cal S}_{i-1}\cup\{j_{s_1},\dots,j_{s_L}\}$
\vspace{0.1cm}\\
4. Perform \ref{eq:aols} $L$ times to update $\{\u_{\ell_1},\dots,\u_{\ell_L}\}_{\ell=1}^{i}$ and $\r_i$
\vspace{0.1cm}\\
5. $\t_j^{(i)} =\t_j^{(i-1)}-\sum_{l = 1}^{L}\frac{{\t_j^{(i-1)}}^\top\u_{i_l}}{\|\u_{i_l}\|_2^2}\u_{i_l}$ for all $j \in \mathcal{I}\backslash \mathcal{S}_i$
\ENDWHILE \vspace{0.1cm}
\STATE 4. $\hat{\x}=\A_{{\cal S}_{i}}^{\dagger}\y$
\end{algorithmic}
\caption{Accelerated Orthogonal Least-Squares (AOLS)}
\label{algo:AOLS}
\end{algorithm}
\section{Performance Analysis of AOLS for sparse recovery}\label{sec:Gua}
In this section, we first study performance of AOLS in the random measurements and noise-free scenario; 
specifically, we consider the linear model \ref{eq:1} where the elements of $\A$ are drawn from 
${\cal N}\left(0,\frac{1}{n}\right)$ and $\e=\mathbf{0}$, and derive conditions for the exact recovery via AOLS. 
Then we generalize this result to the noisy scenario. First, we begin by stating three lemmas later used in the 
proofs of main theorems.  

\subsection{Lemmas}
As stated in Section~1, existing analysis of OMP under Gaussian measurements \cite{tropp2007signal,fletcher2012orthogonal} alter the selection criterion to analyze probability of successfull recovery. On the other hand, current analysis of OLS in \cite{foucart2012stability} relies on the assumption that the coefficient matrix has normalized columns and cannot be directly applied to the case of Gaussian measurement, the scenario considered in this paper for analysis of the proposed AOLS algorithm. As part of our contribution, we provide
Lemma \oldref{lem:23} that states the projection of a random vector drawn from a zero-mean Gaussian 
distribution onto a random subspace preserves its expected Euclidean norm (within a normalizing factor which 
is a function of the problem parameters) and is with high probability concentrated around its expected value.
\begin{lemma} \label{lem:23}
\textit{Assume that an $n \times m$ coefficient matrix $\A$ consists of entries that are drawn independently
from ${\cal N}(0,1/n)$
and let $\A_k\in\R^{n\times k}$ be a submatrix of $\A$. 
Then, $\forall {\bf u}\in \R^{n}$ statistically independent of $\A_k$ drawn
according to $\u \sim {\cal N}(0,1/n)$,
it holds that
$\E\norm{\P_k\u}_2^2=\frac{k}{n}\E\norm{\u}_2^2$. Moreover, let $c_0(\epsilon)=\frac{\epsilon^2}{4}-\frac{\epsilon^3}{6}$.
Then,
\begin{equation}
\Pr\{(1-\epsilon)\frac{k}{n}<\norm{\P_k\u}_2^2<(1+\epsilon)\frac{k}{n}\}\geq 1-2e^{-kc_0(\epsilon)}.
\end{equation}} 
\end{lemma}
\begin{proof}
See Appendix \oldref{pf:lem2}.  
\end{proof}
Lemma \oldref{lem:nois} (Corollary 2.4.5 in \cite{golub2012matrix}) states inequalities between the 
maximum and minimum singular values of a matrix and its submatrices.
\begin{lemma}\label{lem:nois} 
\textit{Let $\A$, $\B$, and $\C$ be full rank tall matrices such that $\C=[\A,\B]$. Then}
\begin{subequations}
\begin{equation} \label{eq:sigA}
\begin{aligned} 
&\sigma_{\min}\left(\A\right)\geq \sigma_{\min}\left(\C\right), &\sigma_{\max}\left(\A\right)\leq \sigma_{\max}\left(\C\right),
\end{aligned}
\end{equation}
\begin{equation}
\begin{aligned}
&\sigma_{\min}\left(\B\right)\geq \sigma_{\min}\left(\C\right), &\sigma_{\max}\left(\B\right)\leq \sigma_{\max}\left(\C\right).
\end{aligned}
\end{equation}
\end{subequations}
\end{lemma}
Lemma \oldref{lem:baranuk} (Lemma 5.1 in \cite{baraniuk2008simple}) estabishes bounds on the singular values of $\A_k$, i.e., a submatrix of $\A$ with $k$ columns.
\begin{lemma}\label{lem:baranuk}
\textit{Let $\A \in \R^{n\times m}$ denote a matrix with entries that 
are drawn independently from ${\cal N}(0,1/n)$.
Then, for any $0 <\delta< 1$ and for all $\x \in \mathrm{Range}(\A_k)$, it holds that}
\begin{equation}
\Pr\{|\frac{\|\A_k\x\|_2}{\|\x\|_2}-1|\leq \delta\}\geq 1-2(\frac{12}{\delta})^ke^{-nc_0(\frac{\delta}{2})}.
\end{equation}
\end{lemma}
\subsection{Noiseless measurements}
In this section we analyze the performance of AOLS when $\v = \mathbf{0}$. The following theorem establishes that when the coefficient matrix consists of entries drawn from ${\cal N}(0,1/n)$
and the measurements are noise-free, AOLS with high probability recovers an unknown sparse vector from the linear combinations of its entries in at most $k$ iterations. 
\begin{theorem}\label{thm:1}
\textit{Suppose $\x \in \R^m$ is an arbitrary sparse vector with $k < m$ non-zero entries. Let $\A \in \R^{n\times m}$
be a random matrix with entries drawn independently from ${\cal N}(0,1/n)$.
Let $\Sigma$ denote an event wherein given noiseless measurements 
$\y=\A\x$, AOLS can recover $\x$ in at most $k$ iterations. Then $\Pr\{\Sigma\}\geq p_1p_2p_3$, where
\begin{equation} \label{eq:probnonois}
\begin{aligned}
p_1&=\left(1-2e^{-(n-k+1)c_0(\epsilon)}\right)^2, \\
p_2&=1-2(\frac{12}{\delta})^ke^{-nc_0(\frac{\delta}{2})}, \mbox{ and }\\
p_3&=\left(1-\sum_{i=0}^{k-1} e^{-\frac{n}{k-i}\frac{1-\epsilon}{1+\epsilon} (1-\delta)^2 }\right)^{m-k-L+1},
\end{aligned}
\end{equation}
for any $0<\epsilon<1$ and $0<\delta<1$.}
\end{theorem}
\begin{proof}
As stated, the proof is inspired by the inductive framework first introduced in 
\cite{tropp2007signal}.\footnote{Our analysis relies on \ref{eq:nols} rather than the computationally 
efficient recursions in \ref{eq:normqq}. Nonetheless, we have shown the equivalence between the 
two criteria in Theorem \oldref{thm:3}.} We can assume, without a loss of generality, that the nonzero 
components of $\x$ are in the first $k$ locations. This implies that $\A$ can be written as 
$\A=\left[\bar{\A} \;\; \widetilde{\A}\right]$, where $\bar{\A} \in \R^{n\times k}$ has columns with 
indices in ${\cal S}_{true}$ and $\widetilde{\A}\in \R^{n\times (m-k)}$ has columns with indices in
${\cal I}\backslash {\cal S}_{true}$. For ${\cal T}_1 \subset {\cal I}$ and ${\cal T}_2 \subset {\cal I}$ 
such that ${\cal T}_1 \cap {\cal T}_2=\oldemptyset$, define
\begin{equation} \label{eq:bforols}  
\b_{j}^{{\cal T}_1}=
\frac{\a_j}{\norm{\P_{{\cal T}_1}^\bot \a_j}_2},\hspace{0.2cm}  j\in {\cal T}_2,
\end{equation}
where $\P_{{\cal T}_1}^\bot$ denotes the projection matrix onto the orthogonal complement of the subspace spanned 
by the columns of $\A$ with indices in ${\cal T}_1$. Using the notation of \ref{eq:bforols}, \ref{eq:nols} becomes
\begin{equation}\label{eq:srp}
j_s=\argmax_{j \in {\cal I}\backslash {\cal S}_{i-1}}{\left|\r_{i-1}^\top \b_{j}^{{\cal S}_{i-1}}\right|}.
\end{equation} 
In addition, let $\F_{{\cal S}_i}=[\b_{j}^{{\cal S}_i}]\in \R^{n\times(k-i)}$, $j\in {\cal S}_{true}\backslash{\cal S}_i$, and 
$\Si_{{\cal S}_i}=[\b_{j}^{{\cal S}_i}]\in\R^{n\times(m-k)}$, $j\in {\cal I}\backslash{\cal S}_{true}$. Assume that in the
first $i$ iterations AOLS selects columns from ${\cal S}_{true}$. Let 
$|\psi_{o_1}^\top\r_i|\leq \dots \leq |\psi_{o_{m-k}}^\top\r_i|$ be an ordering of the set 
$\{|\psi_1^\top\r_i|, \dots,|\psi_{m-k}^\top\r_i|\}$. According to the selection rule in \ref{eq:srp}, AOLS identifies at 
least one true column in the $(i+1)\ts{st}$ iteration if the maximum correlation between $\r_i$ and columns of 
$\F_{{\cal S}_i}$  is greater than the  $|\mathcal{P}(\Si_{{\cal S}_i}^\top \r_i)_{m-k-L+1}|$. Therefore, 
\begin{equation}
\rho(\r_i)=\frac{|\mathcal{P}(\Si_{{\cal S}_i}^\top \r_i)_{m-k-L+1}|}{\|\F_{{\cal S}_i}^\top \r_i\|_\infty}<1
\end{equation}
guarantees that AOLS selects at least one true column in the $(i+1)\ts{st}$ iteration. Hence, $\rho(\r_i)<1$ for $i \in \{0,\dots,k-1\}$ 
ensures recovery of $\x$ in $k$ iterations. In other words, $\max_{i}{\rho(\r_i)}<1$ is sufficient condition for AOLS to 
successfully recover the support of $\x$, i.e., if $\Sigma$ denotes the event that AOLS succeeds, then 
$\Pr\{\Sigma\}\geq \Pr\{\max_{i}{\rho(\r_i)}<1\}$. We may upper bound $\rho(\r_i)$ as
\begin{equation}
\rho(\r_i)\leq \frac{|\mathcal{P}(\widetilde{\A}^\top \r_i)_{m-k-L+1}|}{\|\bar{\A}^\top \r_i\|_\infty}\frac{\max_{j\in{\cal S}_{true}}{\|\P_i^\bot\a_j\|_2}}{\min_{j\not\in{\cal S}_{true}}{\|\P_i^\bot\a_j\|_2}}.
\end{equation}
According to Lemma \oldref{lem:23},
\begin{equation} \label{eq:thm11}
\begin{aligned}
\rho(\r_i)&\leq \frac{|\mathcal{P}(\widetilde{\A}^\top \r_i)_{m-k-L+1}|}{\|\bar{\A}^\top \r_i\|_\infty}\sqrt{\frac{1+\epsilon}{1-\epsilon}}\sqrt{\frac{(n-i)\slash n}{(n-i)\slash n}}\frac{\E\|\a_{j_{\max}}\|_2}{\E\|\a_{j_{\min}}\|_2}\\
&=\sqrt{\frac{1+\epsilon}{1-\epsilon}}\frac{|\mathcal{P}(\widetilde{\A}^\top \r_i)_{m-k-L+1}|}{\|\bar{\A}^\top \r_i\|_\infty}
\end{aligned}
\end{equation}
with probability exceeding $p_1=\left(1-2e^{-(n-k+1)c_0(\epsilon)}\right)^2$ for $0\leq i<k$. Let 
$c_1(\epsilon)=\sqrt{\frac{1-\epsilon}{1+\epsilon}}$. Using a simple norm 
inequality and exploiting the fact that $\bar{\A}^\top \r_i$ has at most $k-i$ nonzero entries leads to
\begin{equation} \label{eq:thm1}
\rho(\r_i)\leq \frac{\sqrt{k-i}}{c_1(\epsilon)}\frac{|\mathcal{P}(\widetilde{\A}^\top \r_i)_{m-k-L+1}|}{\|\bar{\A}^\top \r_i\|_2}
=\frac{\sqrt{k-i}}{c_1(\epsilon)}\|\widetilde{\A}^\top \widetilde{\r}_i\|_\infty,
\end{equation}
where $\widetilde{\r}_i=\r_i\slash \|\bar{\A}^\top \r_i\|_2$.
According to Lemma \oldref{lem:baranuk}, for any $0<\delta<1$, $\Pr\{\|\widetilde{\r}_i\|_2\leq \frac{1}{1-\delta}\}\geq 1-2(\frac{12}{\delta})^ke^{-nc_0(\frac{\delta}{2})}=p_2$. Subsequently,
\begin{equation}
\begin{aligned}
\Pr\{\Sigma\}&\geq p_1 p_2 \Pr\{\max_{0\leq i <k}{|\mathcal{P}(\widetilde{\A}^\top \r_i)_{m-k-L+1}|}<c_1(\epsilon)\}\\
& \geq p_1 p_2 \prod_{j=1}^{m-k-L+1}\Pr\{\max_{0\leq i <k}{\left|\widetilde{\a}_{o_j}^\top \widetilde{\r}_i\sqrt{k-i}\right|}<c_1(\epsilon)\}\\
&=p_1 p_2 \Pr\{\max_{0\leq i <k}{\left|\widetilde{\a}_{o_1}^\top \widetilde{\r}_i\sqrt{k-i}\right|}<c_1(\epsilon)\}^{m-k-L+1},
\end{aligned}
\end{equation}
where we used the assumption that the columns of $\widetilde{\A}$ are independent. Note that the random vectors 
$\{\widetilde{\r}_i\sqrt{k-i}\}_{i=0}^{k-1}$ are bounded with probability exceeding $p_2$ and are statistically independent 
of $\widetilde{\A}$. Now, recall that the entries of $\A$ are drawn independently from 
${\cal N}\left(0,\frac{1}{n}\right)$. Since the random variable 
$X_{i}=\widetilde{\a}_{o_1}^\top \widetilde{\r}_i\sqrt{k-i}$ is distributed as ${\cal N}(0,\sigma^2)$ with $\sigma^2\leq \frac{k-i}{n(1-\delta)^2}$, by using a Gaussian 
tail bound and Boole's inequality it is straightforward to show that 
\begin{equation}
\Pr\{\max_{0\leq i <k}{\left|X_{i}\right|< c_1(\epsilon)}\}\geq 1-\sum_{i=0}^{k-1} e^{-\frac{n}{k-i}c_1(\epsilon)^2 (1-\delta)^2 }.
\end{equation}
Thus, 
$\Pr\{\Sigma\} \geq p_1 p_2 p_3 $, where 

\[p_3=\left(1-\sum_{i=0}^{k-1} e^{-\frac{n}{k-i}c_1(\epsilon)^2 (1-\delta)^2 }\right)^{m-k-L+1}.\]
This completes the proof.
\end{proof}
Using the result of Theorem \oldref{thm:1}, one can numerically show that AOLS successfully recovers $k$-sparse $\x$ if the number of measurements is linear in $k$ (sparsity) and logarithmic in $\frac{m}{k+L-1}$.
\begin{corollary}\label{co:1}
\textit{Let $\x \in \R^m$ be an arbitrary $k$-sparse vector and let $\A \in \R^{n\times m}$ denote a matrix with entries that are drawn independently from ${\cal N}(0,1/n)$;
moreover, assume that $n \geq \max\{\frac{6}{C_1} k \log \frac{m}{(k+L-1)\sqrt[3]{\beta}}, \frac{C_2 k+ \log \frac{8}{\beta^2}}{C_3}\}$, where $0<\beta<1$ and $C_1$, $C_2$, and $C_3$ are positive constants independent of $\beta$, $n$, $m$, and  $k$. Given noiseless measurements $\y=\A\x$, AOLS can recover $\x$ in at most $k$ iterations with probability of success exceeding $1-\beta^2$.}
\end{corollary}
\begin{proof}
Let us first take a closer look at $p_3$. Note that $(1-x)^l\geq 1-lx$ is valid for $x\leq 1$ and $l\geq 1$; since replacing $k-i$ 
with $k$ in the expression for $p_3$ in \ref{eq:probnonois} decreases $p_3$, $k(m-k-L+1)\leq \frac{1}{4}(\frac{m}{k+L-1})^6$ for $m>(k+L-1)^{3\slash 2}$ and we obtain
\begin{equation}
p_3\geq 1-\frac{1}{4}(\frac{m}{k+L-1})^6 e^{-C_1\frac{n}{k}},
\label{p3ineq}
\end{equation}
where $C_1=\frac{1-\epsilon}{1+\epsilon} (1-\delta)^2>0$. Multiplying both sides of \ref{p3ineq} with $p_1$ and $p_2$ and 
discarding positive higher order terms leads to
\begin{equation}
\Pr\{\Sigma\}\geq 1-\frac{1}{4}(\frac{m}{k+L-1})^6 e^{-C_1\frac{n}{k}}-2e^{\log \frac{12}{\delta}k}e^{-n c_0(\frac{\delta}{2})}-4e^{c_0(\epsilon)k}e^{-n c_0(\epsilon)}.
\end{equation}
This inequality is readily simplified by defining positive constants $C_2=\max_{0<\epsilon,\delta<1}{\{\log \frac{12}
{\delta},c_0(\epsilon)\}}$ and $C_3=\min_{0<\epsilon,\delta<1}{\{c_0(\frac{\delta}{2}), c_0(\epsilon)\}}$,
\begin{equation}
\Pr\{\Sigma\}\geq 1-\frac{1}{4}(\frac{m}{k+L-1})^6 e^{-C_1\frac{n}{k}}-6e^{C_2 k}e^{-n C_3}.
\end{equation}
We need to show that $\Pr\{\Sigma\}\geq 1-\beta^2$. To this end, it suffices to demonstrate that
\begin{equation} \label{eq:fa}
\beta^2 \geq \frac{1}{4}(\frac{m}{k+L-1})^6 e^{-C_1\frac{n}{k}}+6e^{C_2 k}e^{-n C_3}.
\end{equation}
Let $n\geq \frac{C_2 k+ \log \frac{8}{\beta^2}}{C_3}$.\footnote{This implies $n\geq k$ for all $m$, $n$, and $k$.} 
This ensures $6e^{C_2 k}e^{-n C_3}\leq \frac{3\beta^2}{4}$ and thus gives the desired result. Moreover,
\begin{equation}
n \geq \max\{\frac{6}{C_1} k \log \frac{m}{(k+L-1)\sqrt[3]{\beta}}, \frac{C_2 k+ \log \frac{8}{\beta^2}}{C_3}\}
\end{equation}
guarantees that $\Pr\{\Sigma\}\geq 1-\beta^2$ with $0<\beta<1$.
\end{proof}
\textit{Remark 1:} Note that when $k \to \infty$ (and so do $m$ and $n$), $p_1$, $p_2$, and $p_3$ are very close
to $1$. Therefore, one may assume very small $\epsilon$ and $\delta$ which implies $C_1\approx 1$. 
\subsection{Noisy measurements}
We now turn to the general case of noisy random measurements and study the conditions under which 
AOLS with high probability exactly recovers support of $\x$ in at most $k$ iterations. 
\begin{theorem}\label{thm:nois}
\textit{Let $\x \in \R^m$ be an arbitrary $k$-sparse vector and let $\A \in \R^{n\times m}$ denote a matrix 
with entries that are drawn independently from ${\cal N}(0,1/n)$. 
Given the noisy measurements $\y=\A\x+\e$ where $\|\e\|_2 \leq \epsilon_\e$, and $\e$ is independent of $\A$ and $\x$, if 
$\min_{\x_j \ne 0}{|\x_j|}\geq (1+\delta+t) \epsilon_{\e}$ for any $t >0$, AOLS can recover $\x$ in at most $k$ iterations with probability of success $\P\{\Sigma\}\geq p_1p_2p_3$ where
\begin{equation}\label{bound:noise}
\begin{aligned}
p_1&=\left(1-2e^{-(n-k+1)c_0(\epsilon)}\right)^2\\
p_2&=1-2(\frac{12}{\delta})^ke^{-nc_0(\frac{\delta}{2})}, \mbox{ and }\\
p_3&=\left(1-\sum_{i=0}^{k-1} e^{-\frac{n \frac{1-\epsilon}{1+\epsilon}(1-\delta)^4}{k\left[\frac{1}{(k-i) t^2}+(1+\delta)^2\right]}}\right)^{m-k-L+1}
\end{aligned}
\end{equation}
for any $0<\epsilon<1$, $0<\delta<1$}.
\end{theorem}
\begin{proof}
See Appendix \oldref{pf:thm2}.
\end{proof}
\textit{Remark 2:} If we define $\SNR=\frac{\|\A\x\|_2^2}{\|\e\|_2^2}$, the condition 
$\min_{\x_j \ne 0}{|\x_j|}\geq (1+\delta+t) \epsilon_\e$ implies
\begin{equation}
\SNR \approx k (1+\delta+t)^2,
\end{equation} 
which  suggests that for exact support recovery via OLS, $\SNR$ should scale linearly with sparsity level.
\begin{corollary}\label{co:2}
\textit{Let $\x \in \R^m$ be an arbitrary $k$-sparse vector and let $\A \in \R^{n\times m}$ denote a matrix with entries that 
are drawn independently from ${\cal N}(0,1/n)$;
moreover, assume that $n \geq \max\{\frac{6}{C_1} k \log \frac{m}{(k+L-1)\sqrt[3]{\beta}}, \frac{C_2 k+ \log \frac{8}{\beta^2}}{C_3}\}$ 
where $0<\beta<1$ and $C_1$, $C_2$, and $C_3$ are positive constants that are independent of $\beta$, $n$, $m$, and 
$k$. Given the noisy measurements $\y=\A\x+\e$ where $\e \sim {\cal N}(0,\sigma^2)$ is independent of $\A$ and $\x$, 
if $\min_{\x_j \ne 0}{|\x_j|}\geq C_4 \|\e\|_2$ for some  $C_4>1$, AOLS can recover $\x$ in at most $k$ iterations with probability 
of success exceeding $1-\beta^2$.}
\end{corollary}
\begin{proof}
The proof follows the steps of the proof to Corollary \oldref{co:1}, leading us to constants
$C_1=\frac{1-\epsilon}{1+\epsilon}(1-\delta)^4(1+t^2(1+\delta)^2)^{-1}$, $C_2=\max_{0<\epsilon,\delta<1}{\{\log \frac{12}{\delta},c_0(\epsilon)\}}>0$, $C_3=\min_{0<\epsilon,\delta<1}{\{c_0(\frac{\delta}{2}), c_0(\epsilon)\}}>0$, and 
$C_4=(1+\delta+t)$.
\end{proof}
\textit{Remark 3:} In general, for the case of noisy measurements $C_1$ is smaller than that of the noiseless setting, implying a more demanding sampling requirement for the former.
\section{Simulations}\label{sec:sim}
\subsection{Confirmation of theoretical results}
\begin{figure}[t]
	\includegraphics[width=0.99\linewidth]{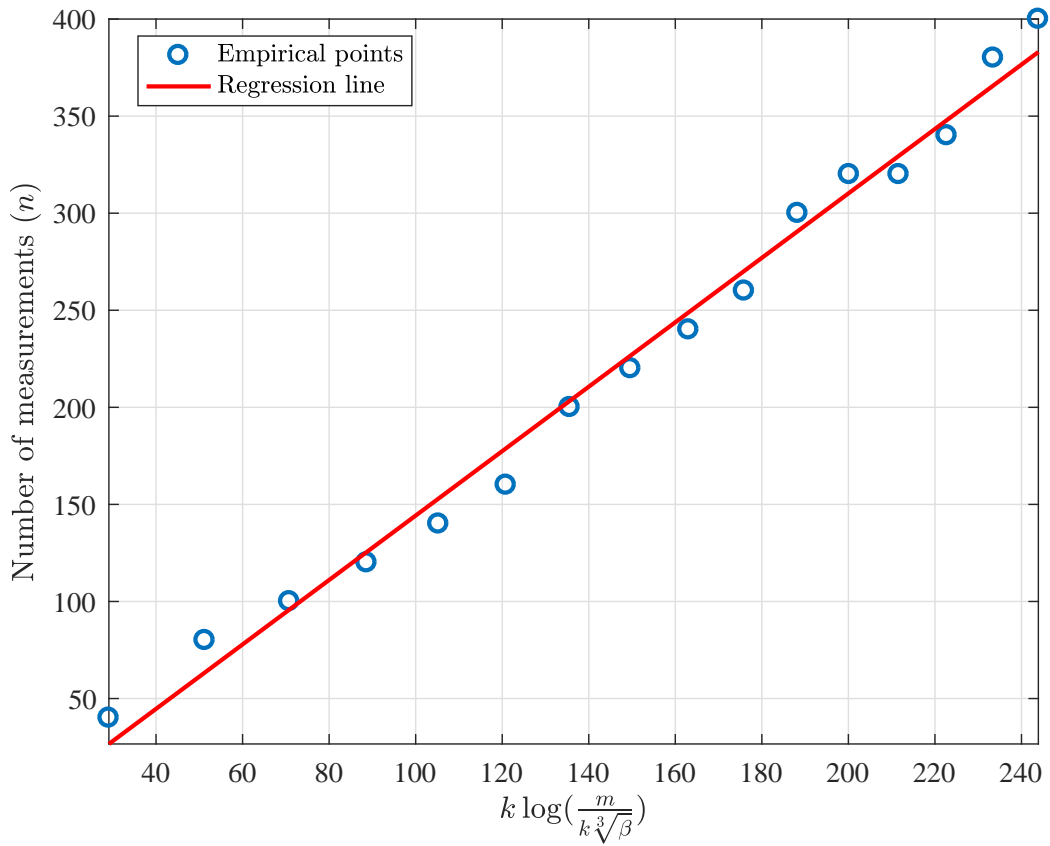}
	\caption{Number of noiseless measurements required for sparse reconstruction with $\beta^2=0.05$ when 
		$m=1024$. The regression line is $n= 2.0109$ $k\log(\frac{m}{k\sqrt[3]{\beta}})$ with the coefficient of determination $R^2=0.9888$.}
	\label{samples:nonoise}
\end{figure}
\begin{figure}[t]
	\includegraphics[width=0.99\linewidth]{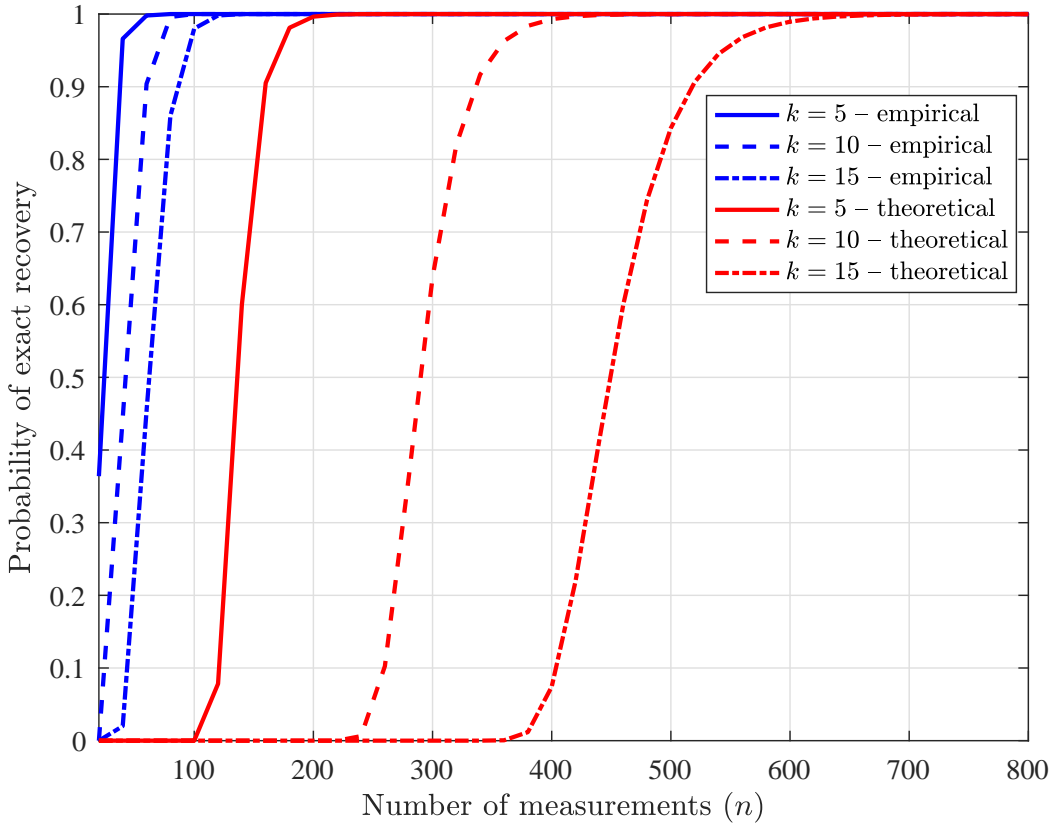}
	\caption{A comparison of the theoretical probability of exact recovery provided by Theorem \oldref{thm:1} 
		with the empirical one, where $m=1024$ and the non-zero elements of $\x$ are drawn independently from
		a normal distribution.}
	\label{probs:nonoise}
\end{figure}
\begin{figure}[t]
	\includegraphics[width=0.99\linewidth]{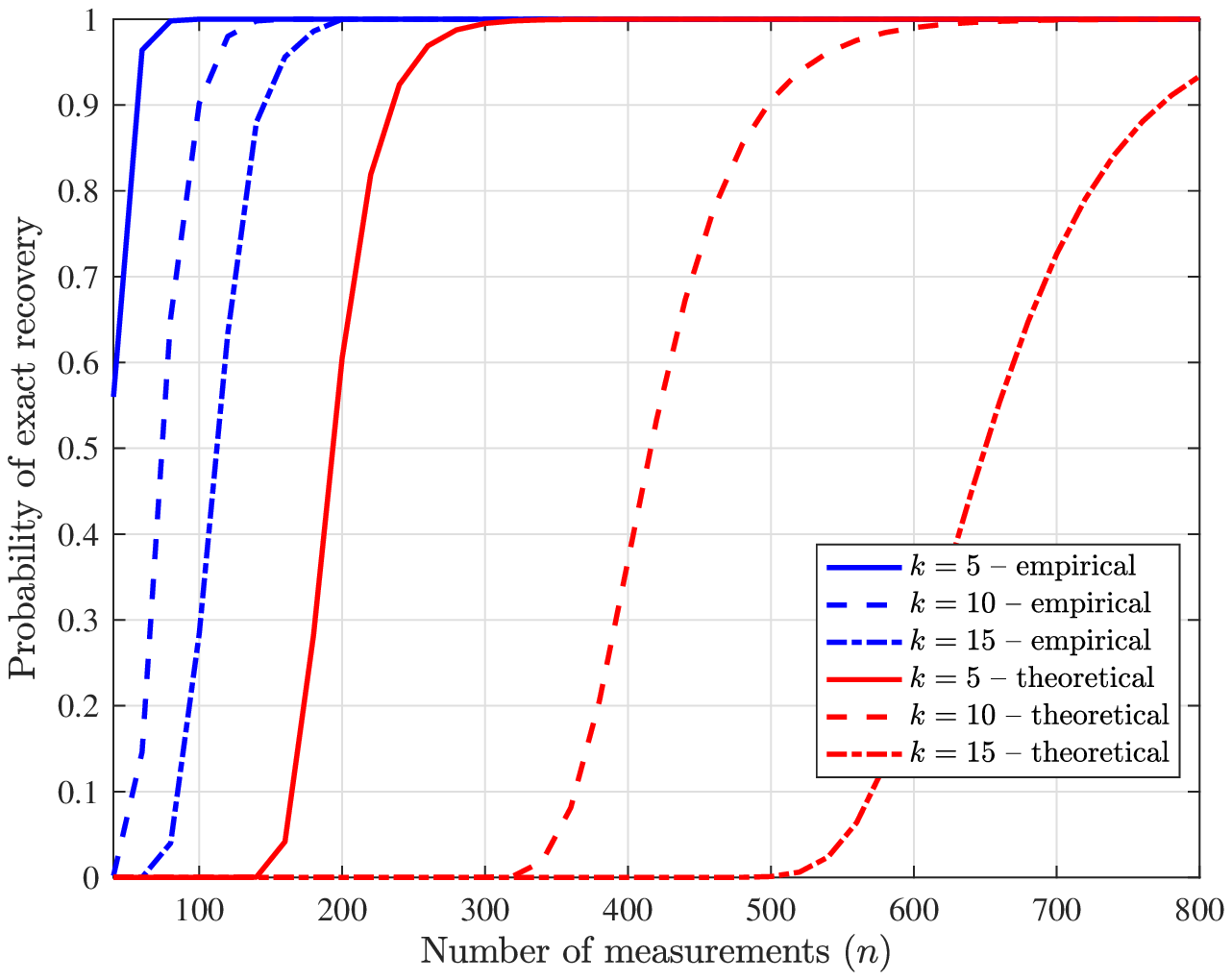}
	\caption{A comparison of the theoretical probability of exact recovery provided by Theorem \oldref{thm:nois} 
		with the empirical one, where $m=1024$ and non-zero elements $\x$ are set to $(1+\delta+20) \|\e\|_2$.}
	\label{bound:noise}
\end{figure}
In this section, we verify our theoretical results by comparing them to the empirical ones obtained via Monte Carlo simulations. 

First, we consider the results of Corollary \oldref{co:1} with $L=1$. In each 
trial, we select locations of the nonzero elements of $\x$ uniformly at random and draw those elements 
from a normal distribution. Entries
of the coefficient matrix $\A$ are also generated randomly from ${\cal N}(0,\frac{1}{n})$.
Fig. \oldref{samples:nonoise} plots the number of noiseless measurement $n$ needed to achieve at least $0.95$ probability of perfect recovery (i.e., $\beta^2=0.05$) as
a function of $k\log(\frac{m}{k\sqrt[3]{\beta}})$. The length of the unknown vector $\x$ here is set to $m=1024$, and the results (shown as
circles) are averaged over $1000$ independent trials. The solid regression line in Fig. \oldref{samples:nonoise} implies linear relation between $n$
and $k\log(\frac{m}{k\sqrt[3]{\beta}})$ as predicted by Corollary \oldref{co:1}. Specifically, for the considered setting, $n\approx 2.0109$ $k\log(\frac{m}{k\sqrt[3]{\beta}})$. 
Recall that, according to Remark 1, for a high dimensional problem where the exact support recovery has the
probability of success overwhelmingly close to 1, $C_1 \approx 1$; this implies $n\geq 6$ $k\log(\frac{m}{k\sqrt[3]{\beta}})$ for all $m$ and $k$. 
Therefore, Fig. \oldref{samples:nonoise} suggests that our theoretical result is somewhat conservative (which is due to approximations that 
we rely on in the proof of Theorem \oldref{thm:1} and Corollary \oldref{co:1}).

In Fig. \oldref{probs:nonoise}, we compare the lower bound on probability of exact recovery from noiseless random measurements established in Theorem \oldref{thm:1} with empirical results. In particular, we consider the setting where $L=1$, $m=1000$ and the non-zero elements of $\x$ are independent and identically distributed 
normal random variables. For three sparsity levels ($k=5,10,15$) we vary the number of measurements and 
plot the empirical probability of exact recovery, averaged over 1000 independent instances. Fig. \oldref{probs:nonoise} illustrates that the theoretical lower bound established  in  \ref{eq:probnonois} becomes
more tight as the signal becomes more sparse.

Next, we compare the lower bound on probability of exact recovery from noiseless random measurements established in Theorem \oldref{thm:nois} with empirical results. More specifically, $L=1$, $m=1000$, $k=5,10,15$,  and the non-zero elements of $\x$ are set to $(1+\delta+20) \|\e\|_2$ to ensure that the condition of Theorem \oldref{thm:nois} imposed on the smallest nonzero element of $\x$ is satisfied. For this setting, in
Fig. \oldref{bound:noise} the results of Theorem \oldref{thm:nois} are compared with the empirical ones 
(the latter are averaged over 1000 independent instances). 
As can be seen from the figure, the lower bound on probability of successful recovery becomes more accurate
for lower $k$, similar to the results for the noiseless scenario illustrated in Fig. \oldref{probs:nonoise}.
\subsection{Sparse recovery performance comparison}
\begin{figure*}[t]
	\begin{subfigure}[]{0.5\textwidth}
		\centering
		\includegraphics[width=1\textwidth]{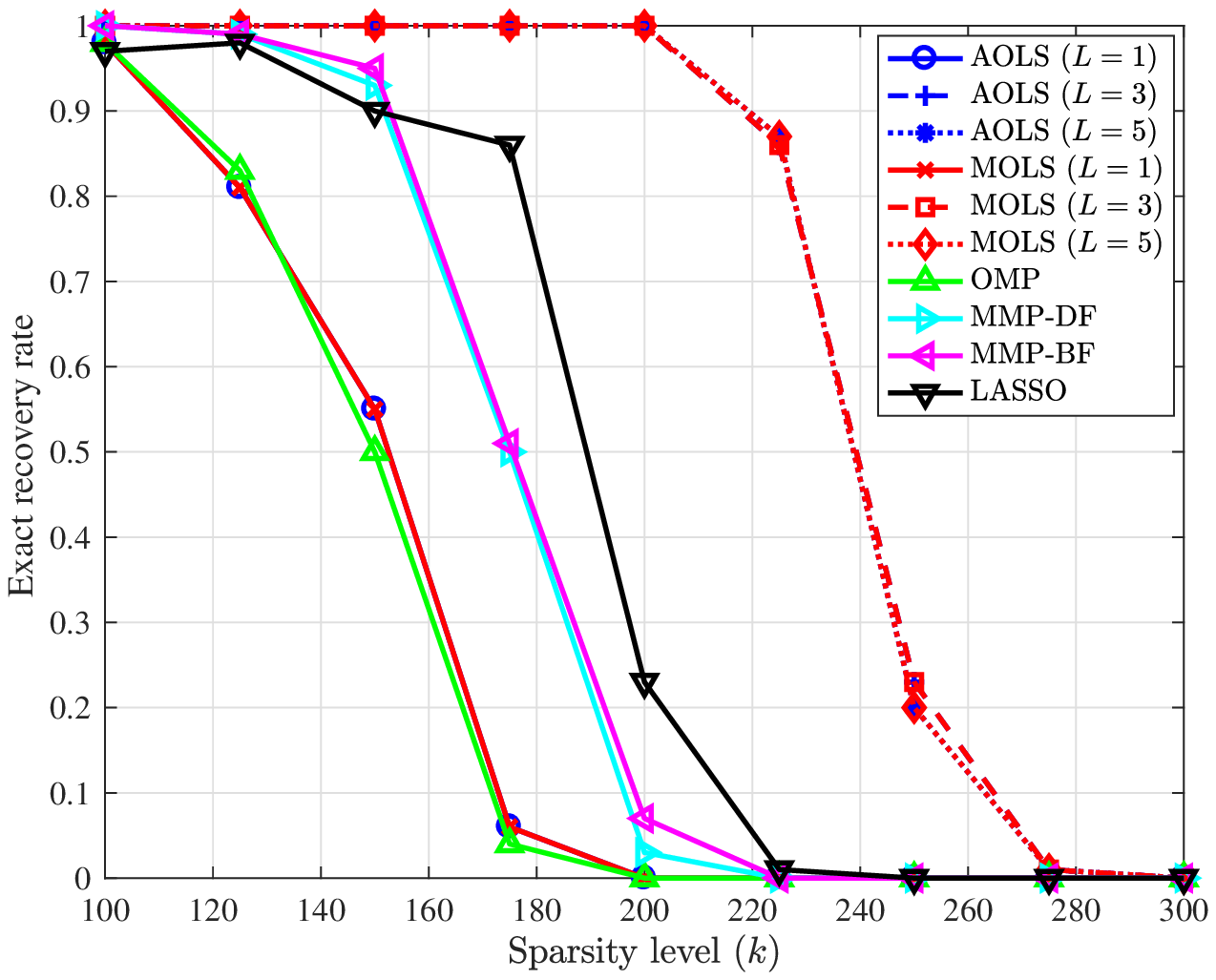}\quad\caption{\footnotesize $T=0$}
	\end{subfigure}
	\begin{subfigure}[]{.5\textwidth}
		\centering
		\includegraphics[width=1\textwidth]{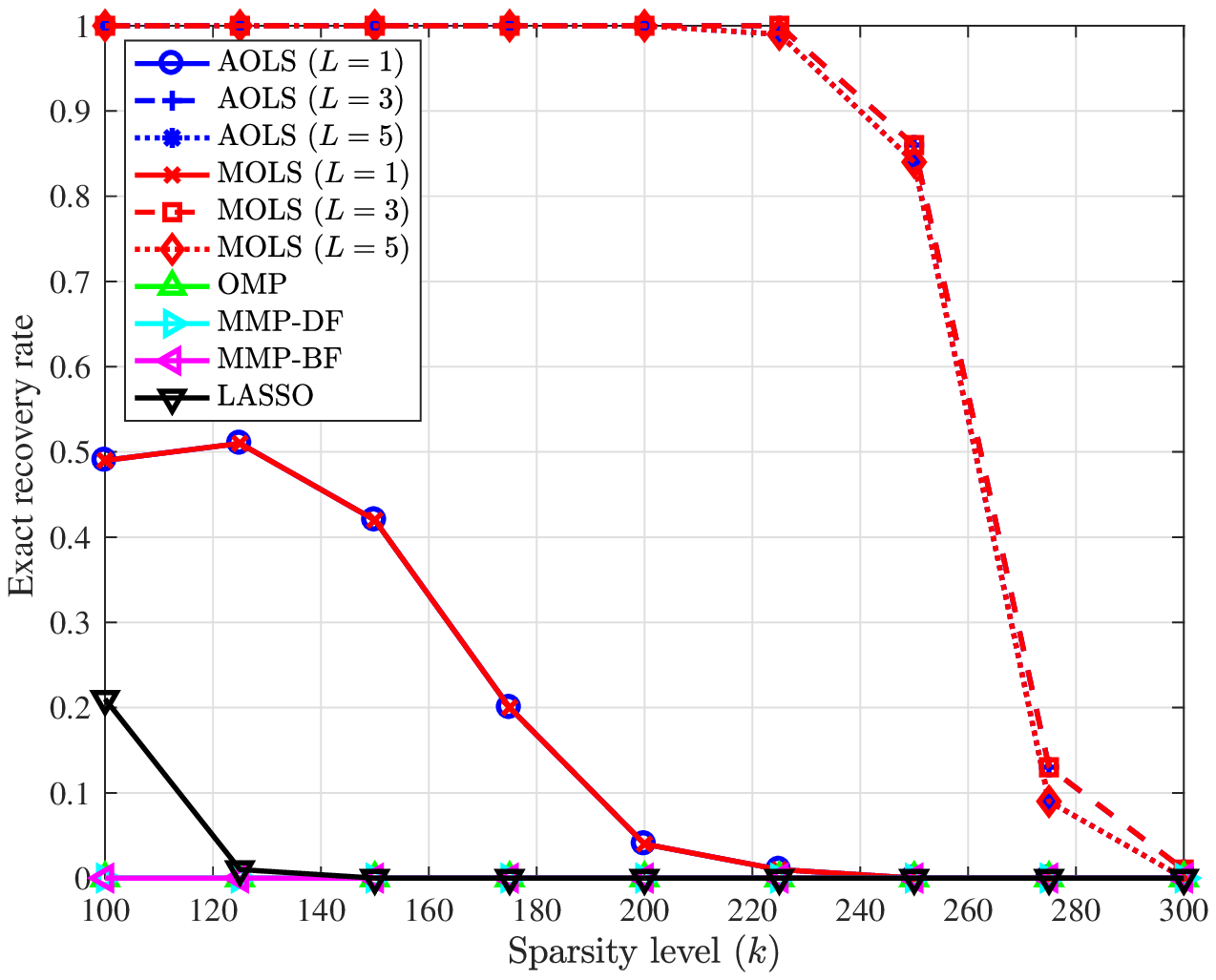}\quad\caption{\footnotesize $T=0.5$}
	\end{subfigure}
	\begin{subfigure}[]{.5\textwidth}
		\centering
		\includegraphics[width=1\textwidth]{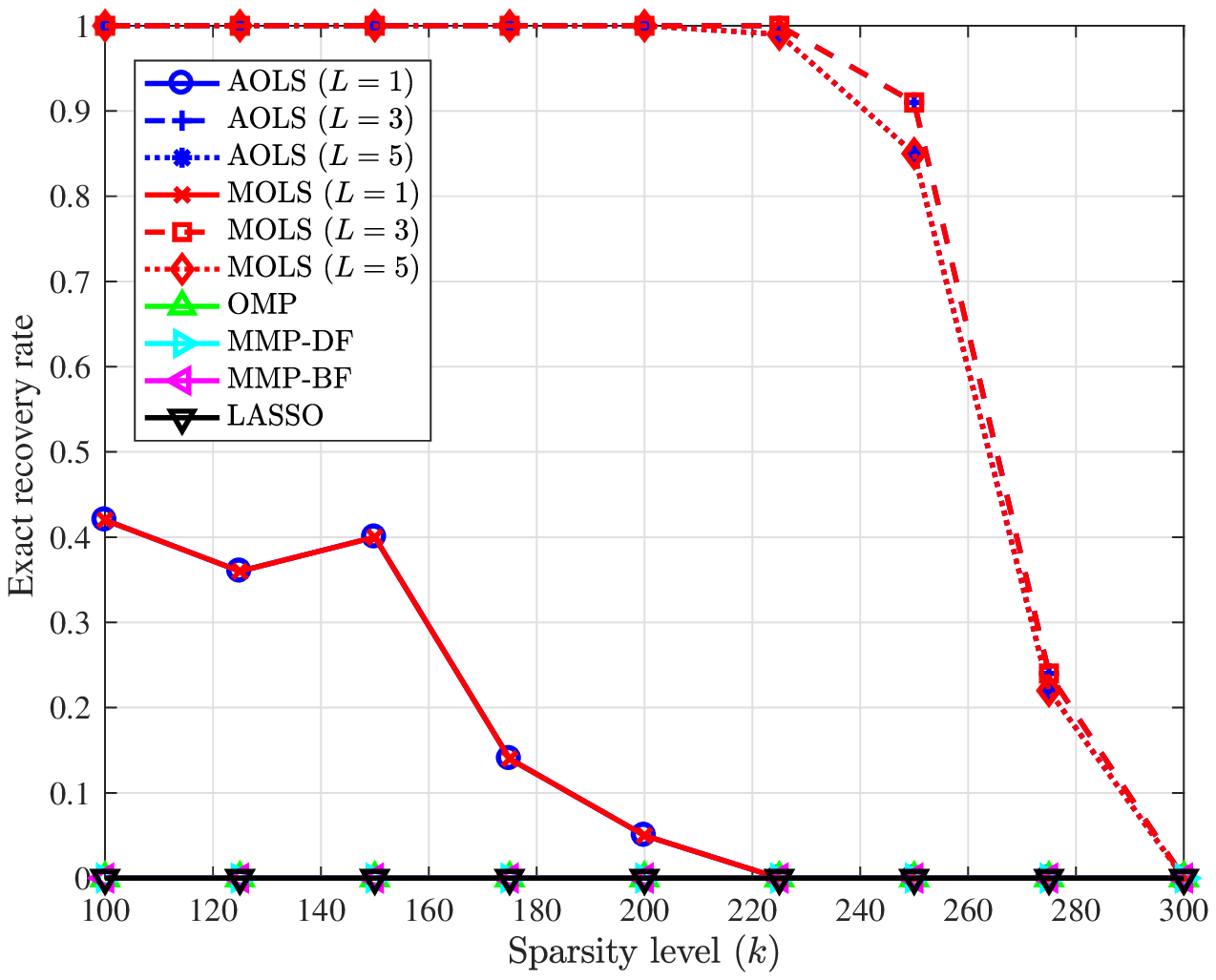}\quad\caption{\footnotesize  $T=1$}
	\end{subfigure}
	\begin{subfigure}[]{.5\textwidth}
		\centering
		\includegraphics[width=1\textwidth]{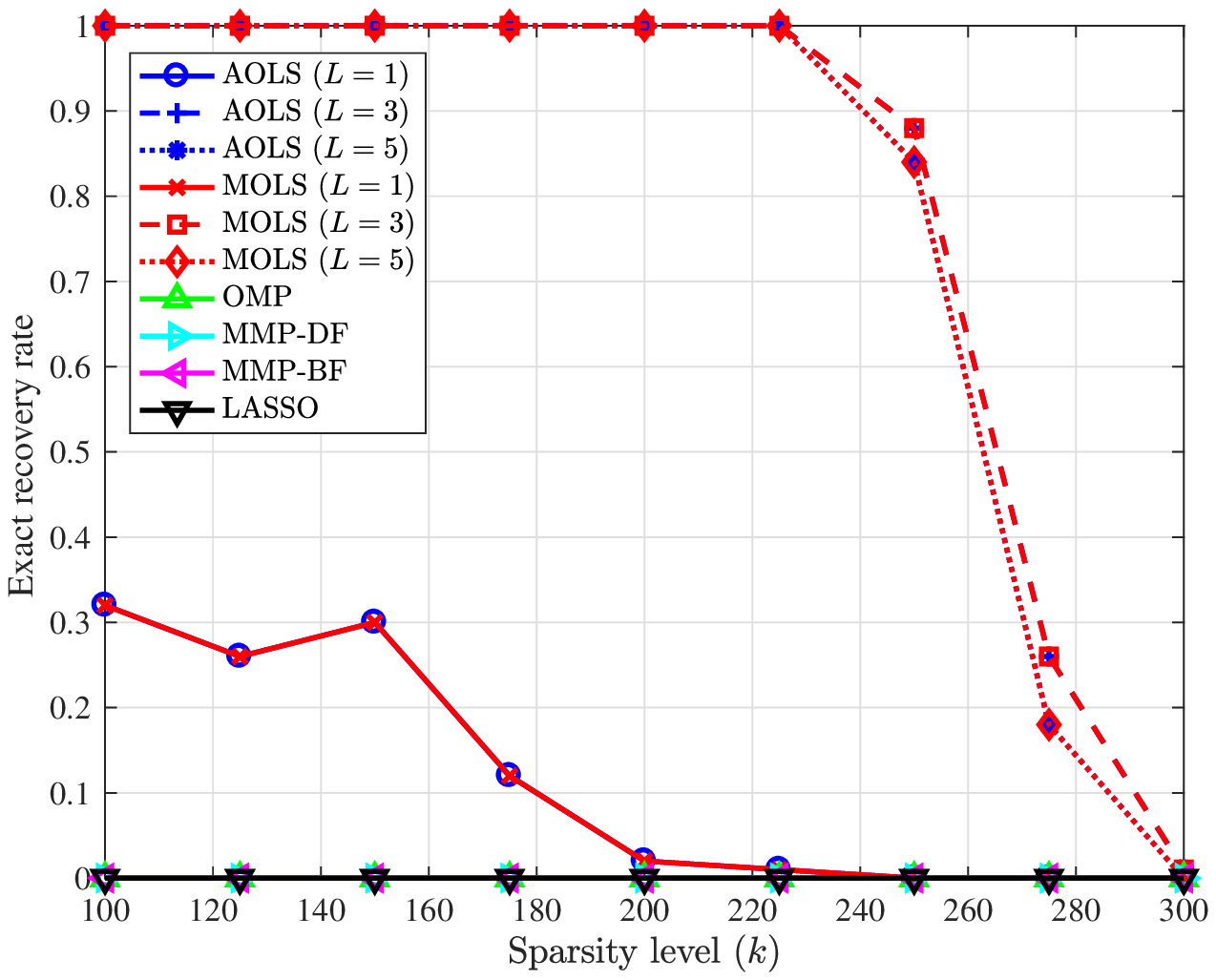}\quad\caption{\footnotesize $T=10$}
	\end{subfigure}
	\caption{\label{fig:ex} Exact recovery rate comparison of AOLS, MOLS, OMP, MMP-DP, MMP-BP, and LASSO for $n=512$, $m=1024$, and $k$ non-zero components of $\x$ uniformly drawn from ${\cal N}(0,1)$ distribution.}   
\end{figure*}
\begin{figure*}[t]
	\begin{subfigure}[]{0.5\textwidth}
		\centering
		\includegraphics[width=1\textwidth]{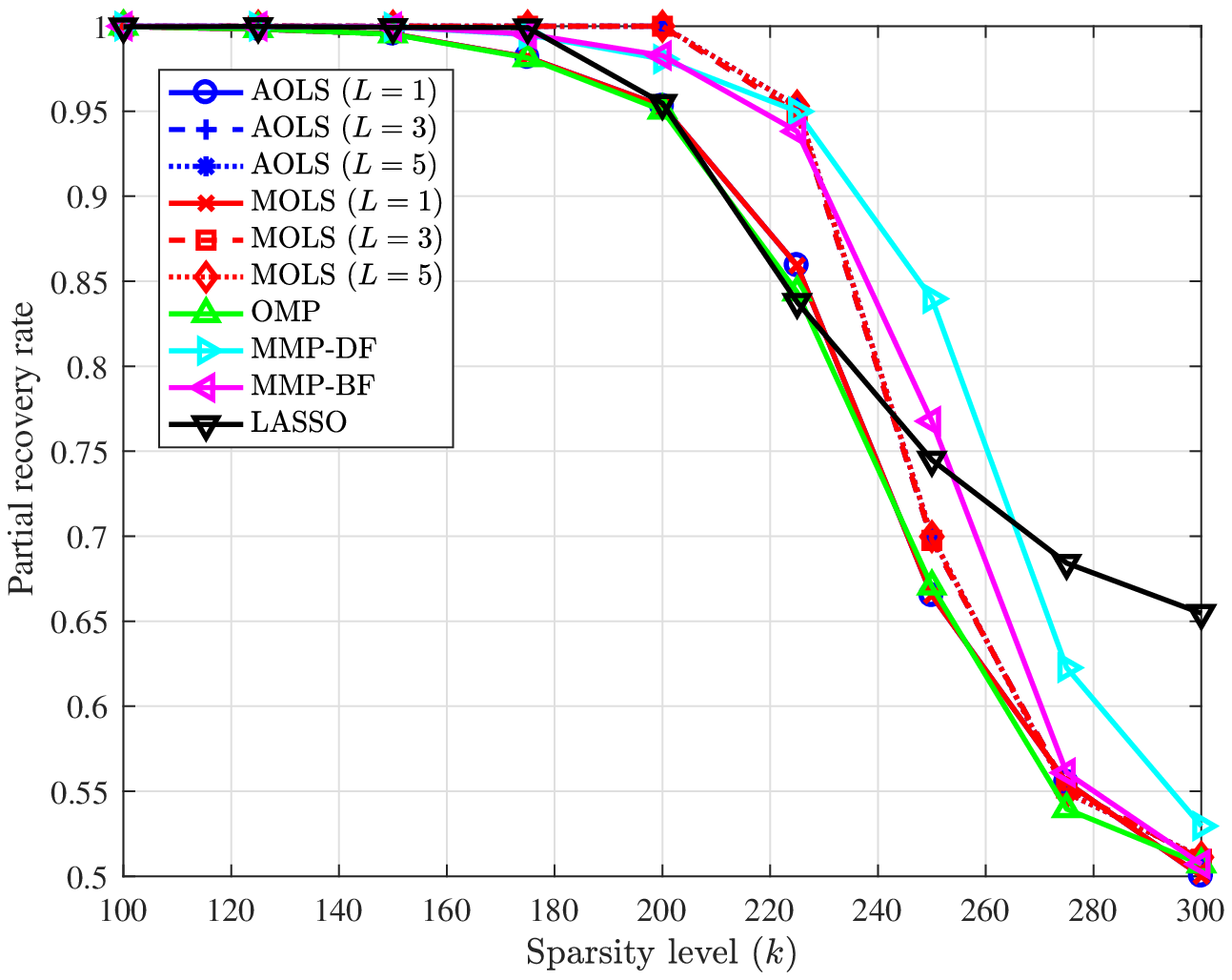}\quad\caption{\footnotesize $T=0$}
	\end{subfigure}
	\begin{subfigure}[]{.5\textwidth}
		\centering
		\includegraphics[width=1\textwidth]{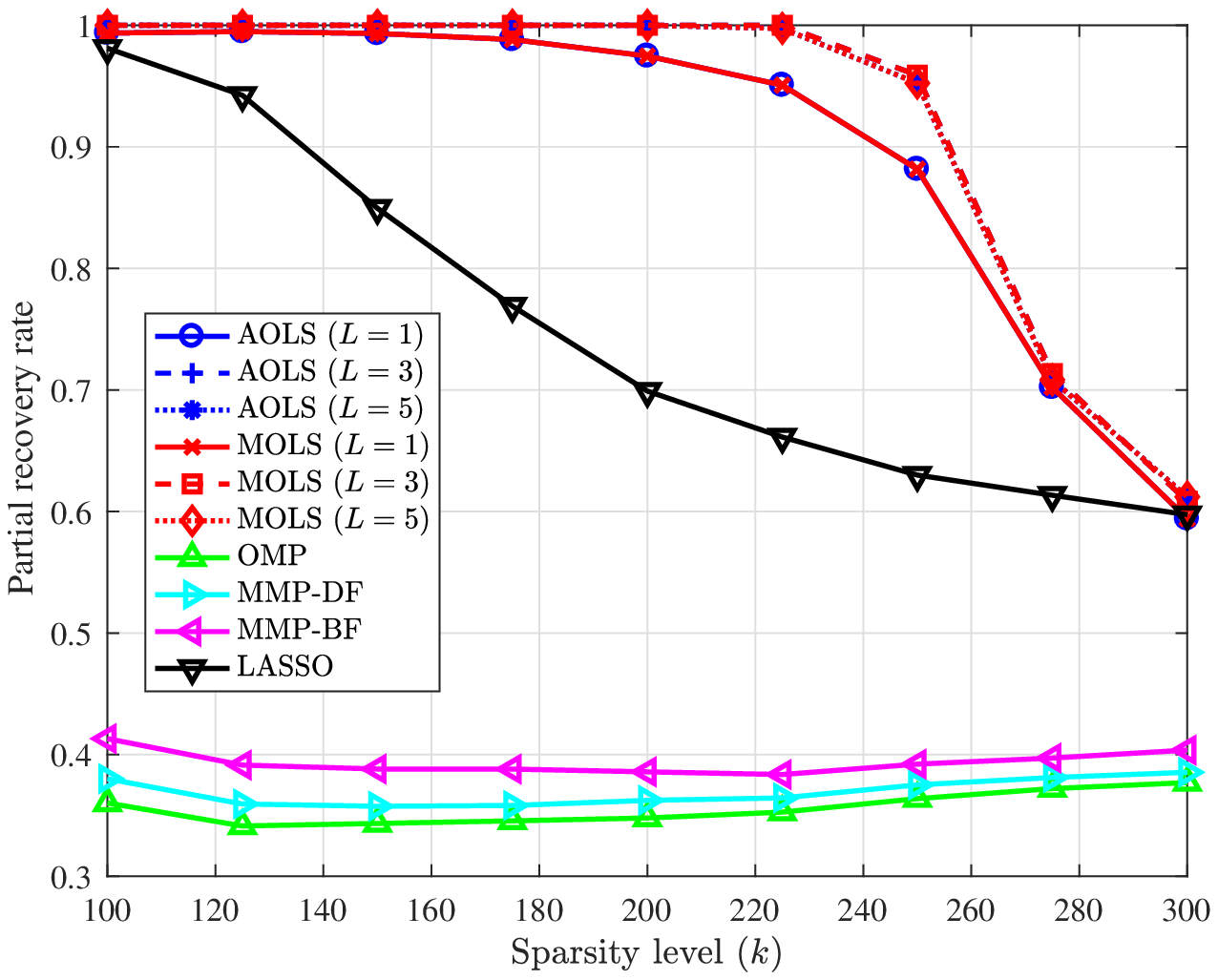}\quad\caption{\footnotesize $T=0.5$}
	\end{subfigure}
	\begin{subfigure}[]{.5\textwidth}
		\centering
		\includegraphics[width=1\textwidth]{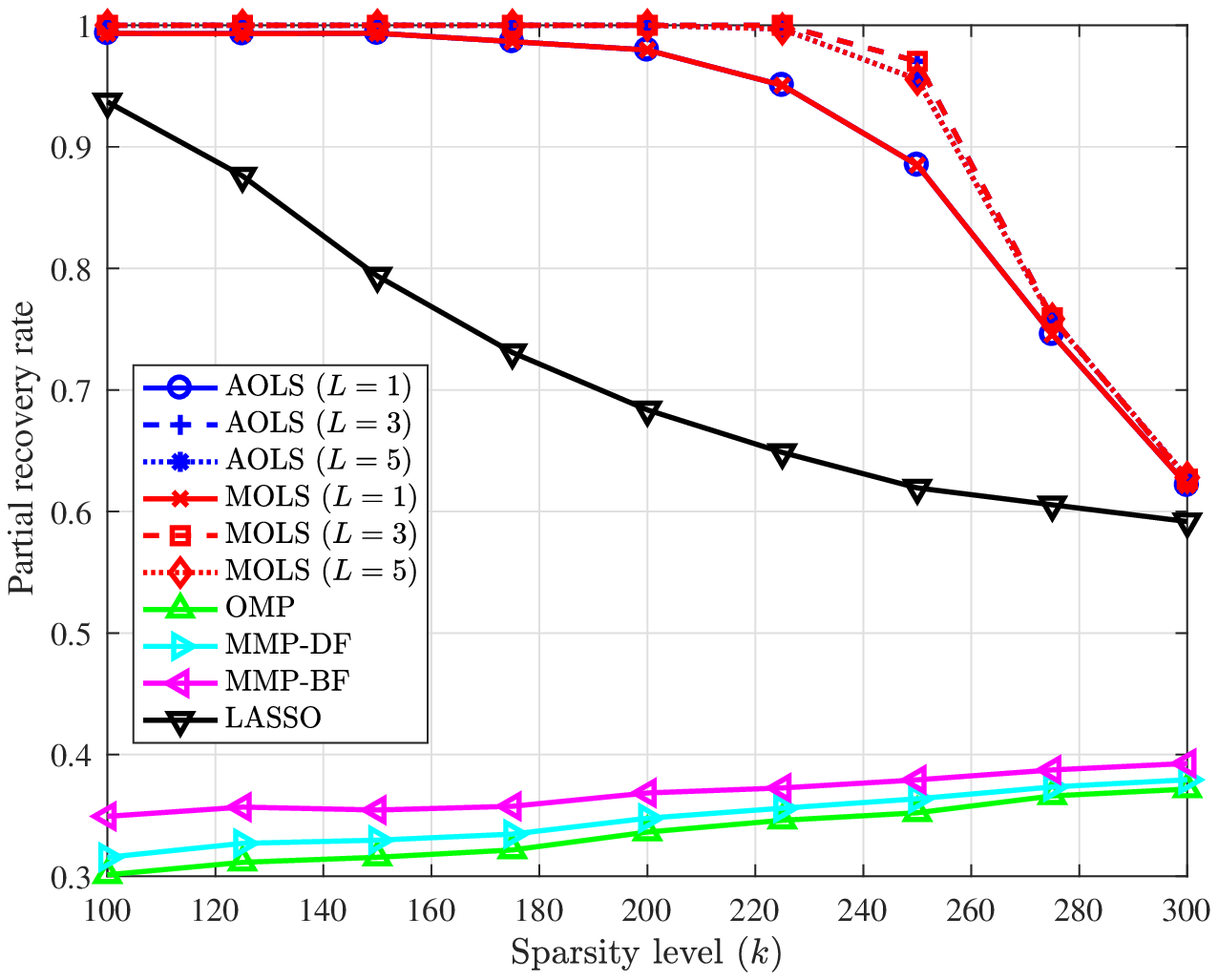}\quad\caption{\footnotesize  $T=1$}
	\end{subfigure}
	\begin{subfigure}[]{.5\textwidth}
		\centering
		\includegraphics[width=1\textwidth]{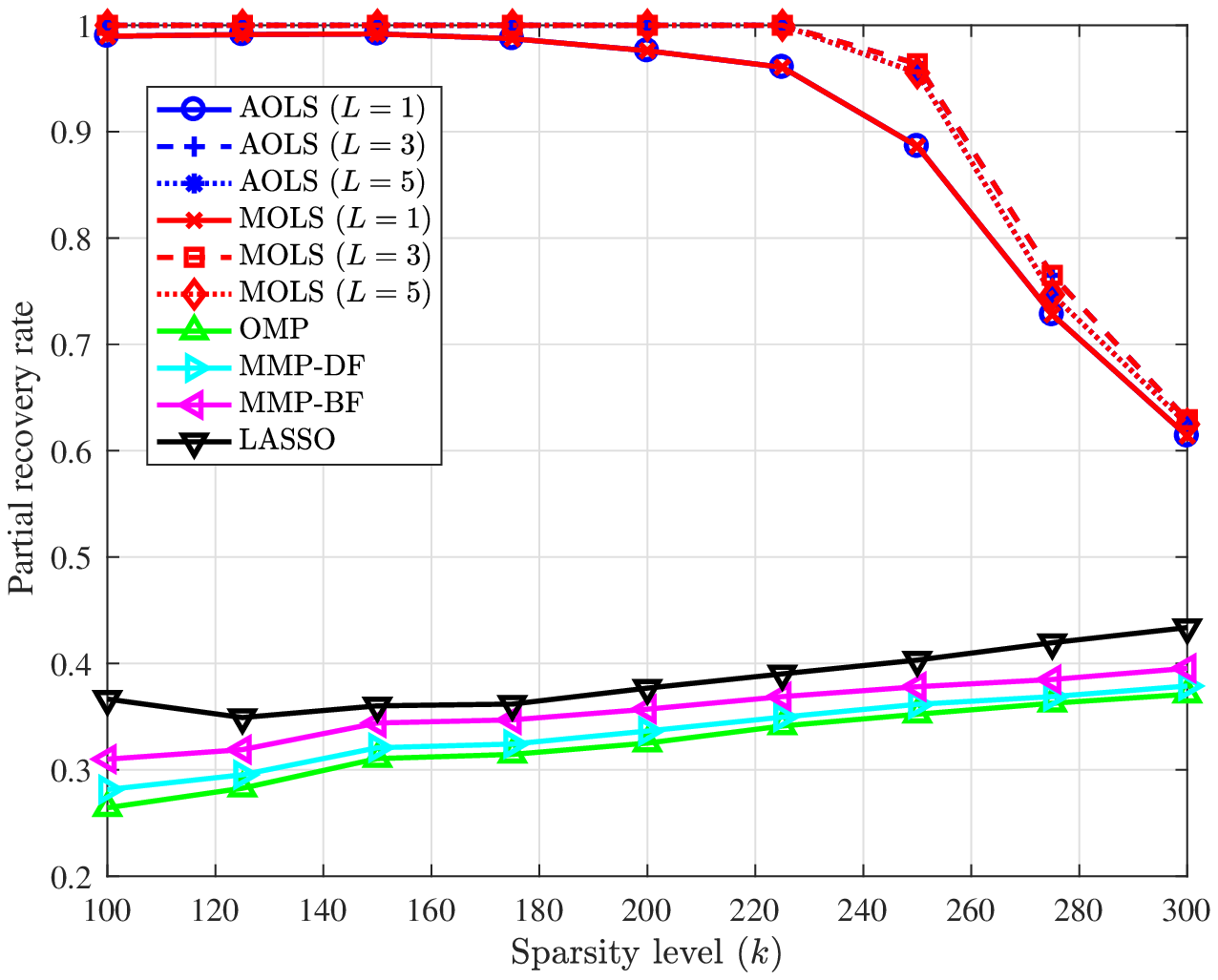}\quad\caption{\footnotesize $T=10$}
	\end{subfigure}
	\caption{\label{fig:rec} Partial recovery rate comparison of AOLS, MOLS, OMP, MMP-DP, MMP-BP, and LASSO for $n=512$, $m=1024$, and $k$ non-zero components of $\x$ uniformly drawn from ${\cal N}(0,1)$ distribution.}   
\end{figure*}
\begin{figure*}[t]
	\begin{subfigure}[]{0.5\textwidth}
		\centering
		\includegraphics[width=1\textwidth]{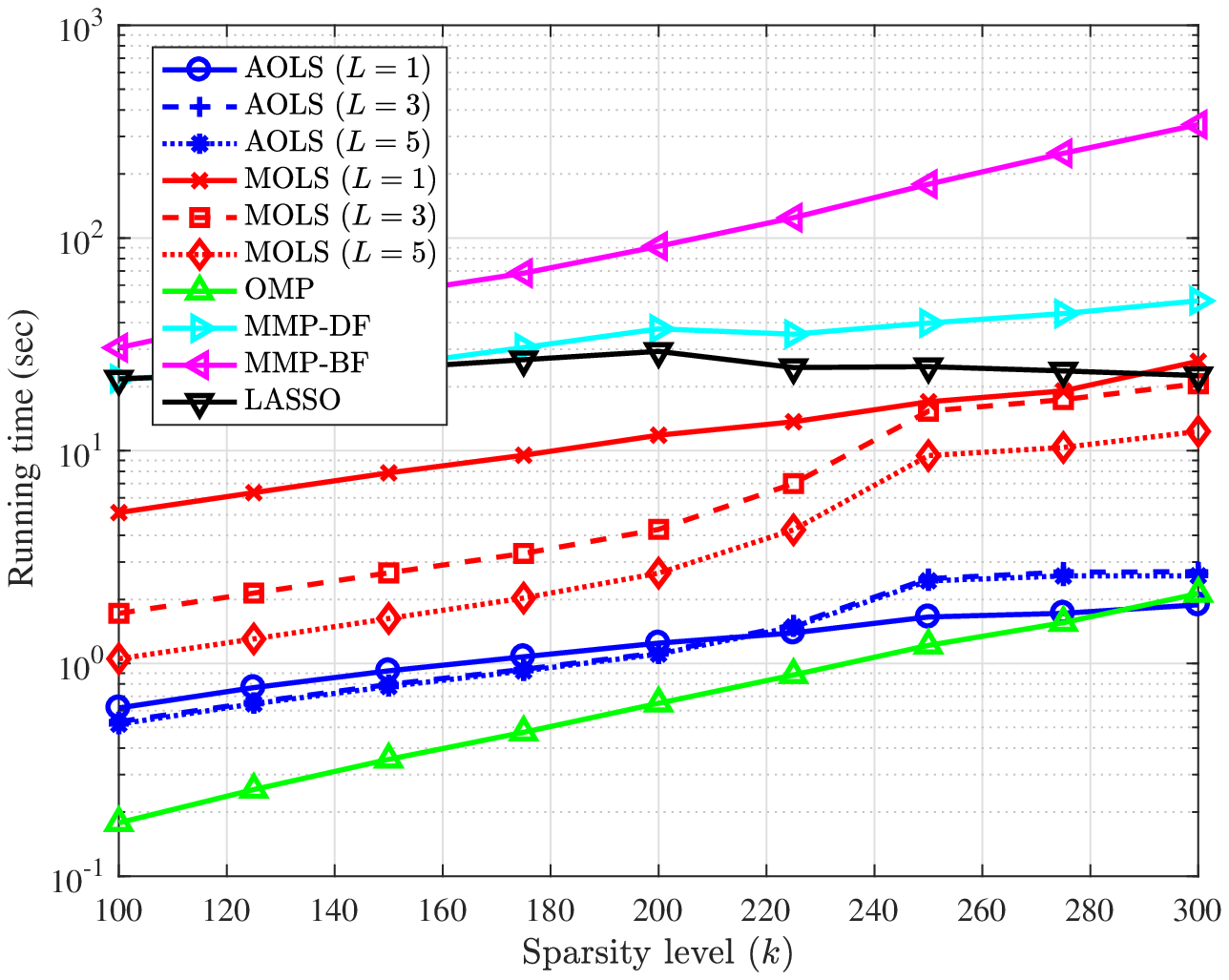}\quad\caption{\footnotesize $T=0$}
	\end{subfigure}
	\begin{subfigure}[]{.5\textwidth}
		\centering
		\includegraphics[width=1\textwidth]{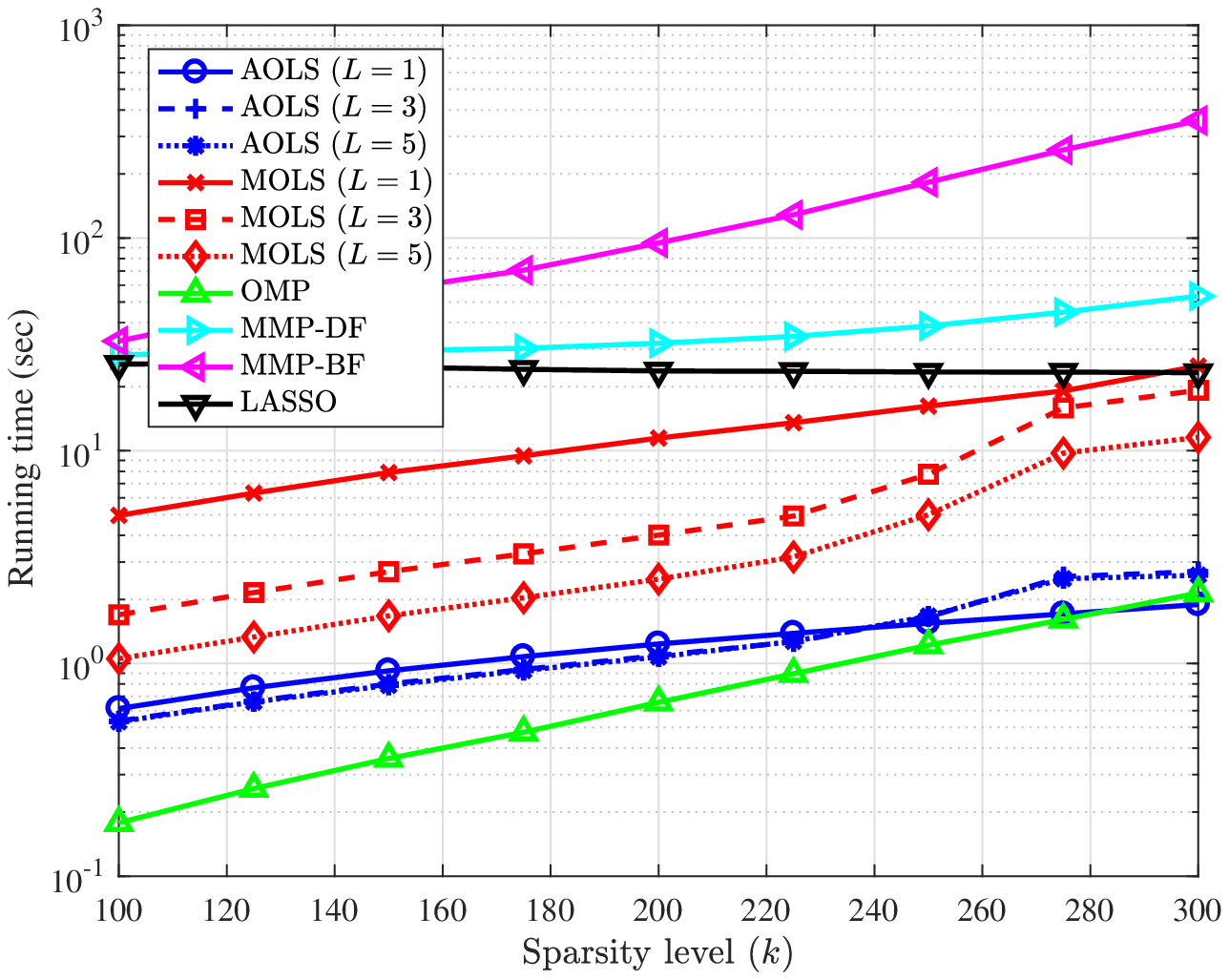}\quad\caption{\footnotesize $T=0.5$}
	\end{subfigure}
	\begin{subfigure}[]{.5\textwidth}
		\centering
		\includegraphics[width=1\textwidth]{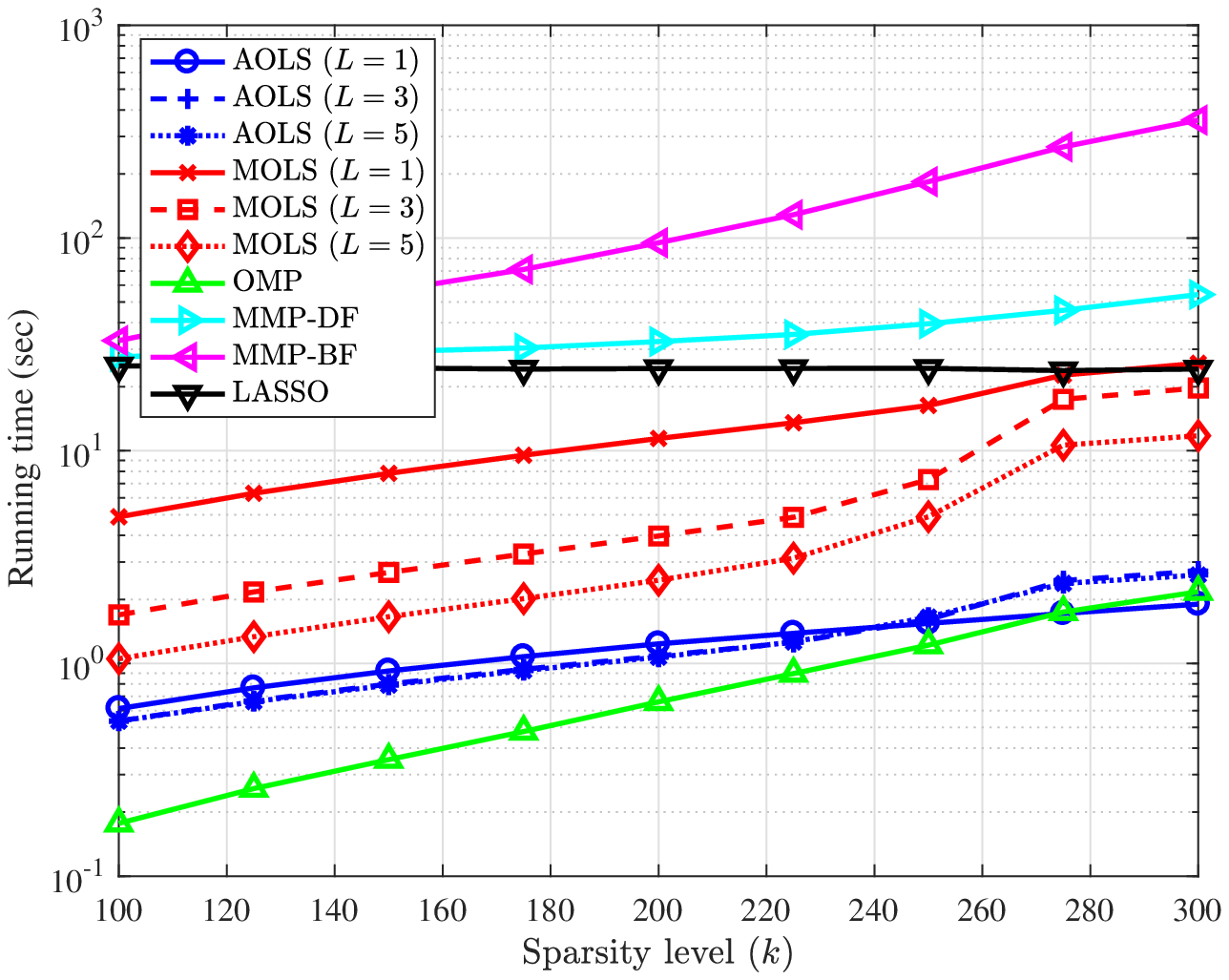}\quad\caption{\footnotesize  $T=1$}
	\end{subfigure}
	\begin{subfigure}[]{.5\textwidth}
		\centering
		\includegraphics[width=1\textwidth]{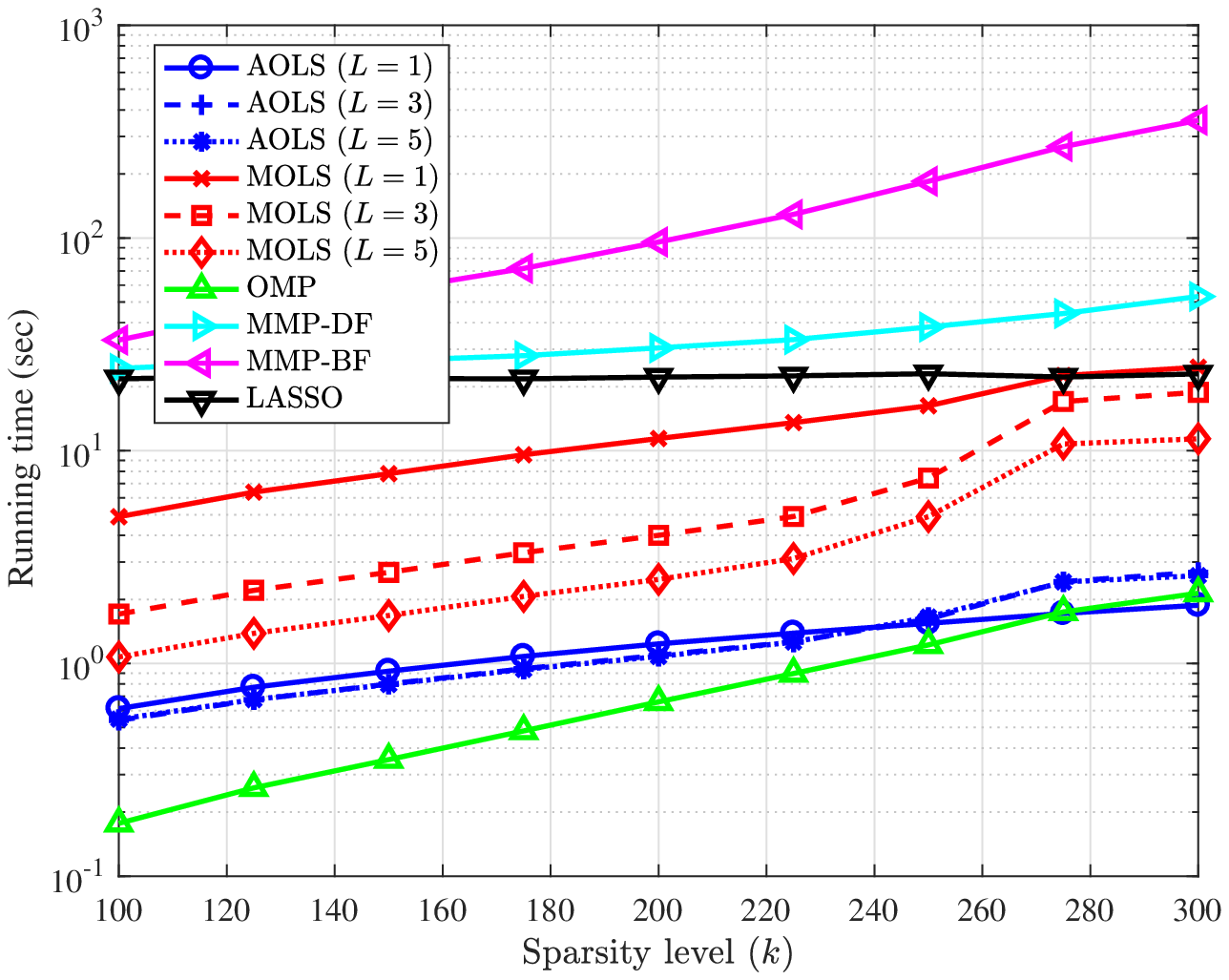}\quad\caption{\footnotesize $T=10$}
	\end{subfigure}
	\caption{\label{fig:time}  A comparison of AOLS, MOLS, OMP, MMP-DP, MMP-BP, and LASSO for $n=512$, $m=1024$, and $k$ non-zero components of $\x$ uniformly drawn from 
	the ${\cal N}(0,1)$ distribution.}   
\end{figure*}
To evaluate performance of the AOLS algorithm, we compare it with that of five state-of-the-art sparse recovery algorithms 
as a function of the sparsity level $k$. In particular, we considered OMP \cite{pati1993orthogonal}, Least Absolute Shrinkage and Selection Operator (LASSO) \cite{tibshirani1996regression,candes2005decoding}, MOLS \cite{wang2017recovery} with $L = 1,3,5$, depth-first and breath-first multipath matching pursiut \cite{kwon2014multipath} (refered to as MMP-DF and MMP-BP, respectively). It is shown in \cite{wang2017recovery,kwon2014multipath} that MOLS, MMP-DF, and MMP-BP outperform many of the sparse recovery algorithms, including OLS \cite{chen1989orthogonal}, OMP \cite{pati1993orthogonal}, GOMP \cite{wang2012generalized}, StOMP \cite{donoho2012sparse}, and BP \cite{candes2005decoding}. Therefore, to demonstrate 
performance and uniqueness of AOLS with respect to other sparse recovery methods, we compare it to these three schemes. We also include performance of OMP and LASSO as baselines.

For MOLS, MMP-DF, and MMP-BF we used the MATLAB implementations provided by the authors of \cite{wang2017recovery,kwon2014multipath}. As typically done in benchmarking 
tests \cite{dai2009subspace,wang2012generalized}, we used CVX \cite{cvx,gb08} to implement the 
LASSO algorithm. We explored various values of $L$ (specifically, $L=1,3,5$) to better understand its effect on the performance of AOLS. The 
stopping threshold for AOLS, MMP-DF, OMP, MOLS, and LASSO was set to $10^{-13}$ (MMP-BF, a breadth-first algorithm, does not use a stopping threshold).

We consider sparse recovery from random measurements in a large-scale setting to fully understand scalability of tested algorithms. To this end, we set $n=512$ and $m=1024$; $k$ changes from $100$ to $300$. The non-zero elements of $\x$ -- whose locations are chosen uniformly 
-- are independent and identically distributed normal random variables. In order to construct $\A$, we consider the so-called 
hybrid scenario \cite{soussen2013joint} to simulate both correlated and uncorrelated dictionaries. 
Specifically, we set 
$\A_j=\frac{\b_j+t_j\mathbf{1}}{\|\b_j+t_j\mathbf{1}\|_2}$ where $\b_j \sim {\cal N}(0,\frac{1}{n})$, $t_j \sim {\cal U}(0,T)$ with 
$T \geq 0$, and $\mathbf{1} \in \R^{n}$ is the all-one vector. In addition, $\{\b_j\}_{j=1}^m$ and $\{t_j\}_{j=1}^m$ are statistically 
independent. Notice that as $T$ increases, the so-called mutual coherence parameter of $\A$ increases, resulting in a more 
correlated coefficient matrix; $T=0$ corresponds to an incoherent $\A$. For each scenario, we use Monte Carlo simulations 
with $100$ independent instances. Performance of each algorithm is characterized by three metrics: (i) Exact Recovery Rate (ERR), 
defined as the percentage of instances where the support of $\x$ is recovered exactly, (ii) Partial Recovery Rate (PRR), measuring 
the fraction of support which is recovered correctly, and (iii) the running time of the algorithm.

The exact recovery rate, partial recovery rate, and running time comparisons are shown in Fig. \oldref{fig:ex}, Fig. \oldref{fig:rec}, and Fig. \oldref{fig:time}, respectively. As can be seen from 
Fig. \oldref{fig:ex}, AOLS and MOLS with $L=3,5$ achieve the best exact recovery rate for various values of $T$. We also observe that as $T$ increases, performance of all schemes, except 
for AOLS and  MOLS, significantly deteriorates. As for the partial recovery rate shown in Fig. \oldref{fig:rec}, for $T=0$ all methods perform similarly. However, AOLS and MOLS are 
robust to changes in $T$ while other schemes perform poorly for larger values of $T$.
Running time comparison results, depicted in Fig. \oldref{fig:time}, demonstrate that for all scenarios the AOLS algorithm is essentially as fast as OMP, while AOLS is significantly more accurate. 
We also observe from the figure that AOLS is significantly faster that other schemes. Specifically, for larger values of $k$, AOLS is around 15 times faster than MOLS, while they deliver essentially 
the same performance. 

\subsection{Application: sparse subspace clustering}
\begin{figure*}[t]
	\begin{subfigure}[]{0.5\textwidth}
		\centering
		\includegraphics[width=1\textwidth]{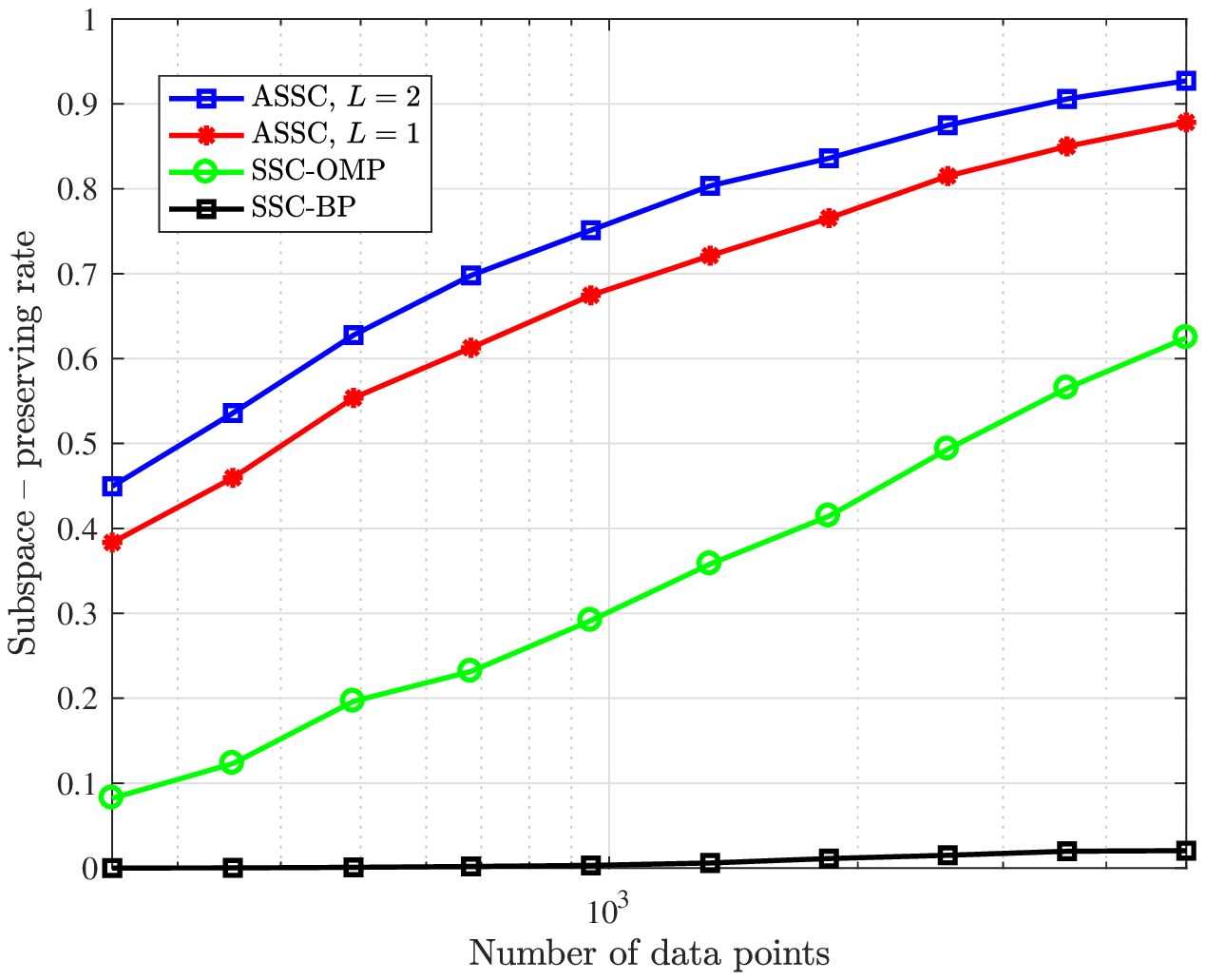}\quad\caption{\footnotesize Subspace preserving rate}
	\end{subfigure}
	\begin{subfigure}[]{.5\textwidth}
		\centering
		\includegraphics[width=1\textwidth]{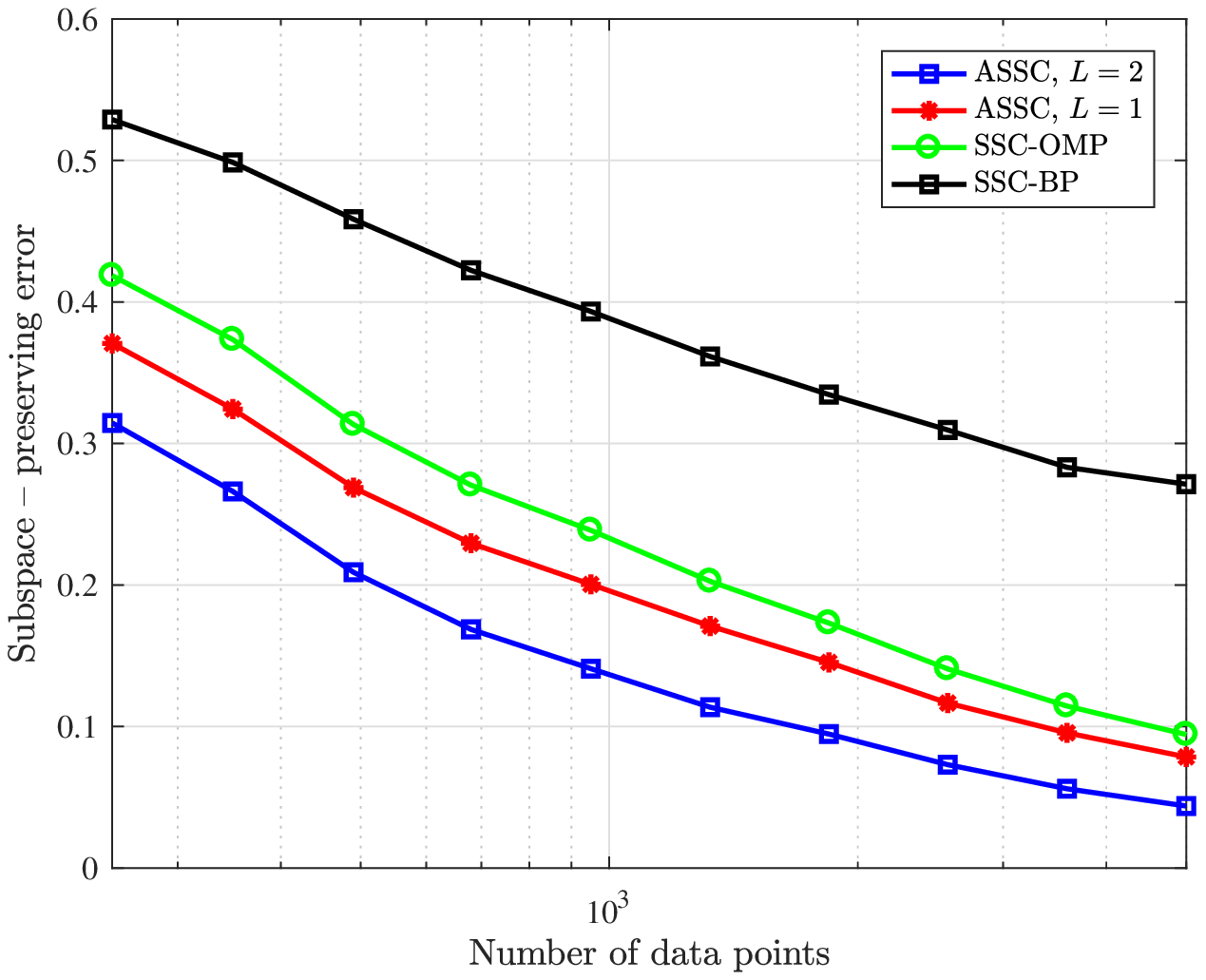}\quad\caption{\footnotesize  Subspace preserving error}
	\end{subfigure}
	\begin{subfigure}[]{.5\textwidth}
		\centering
		\includegraphics[width=1\textwidth]{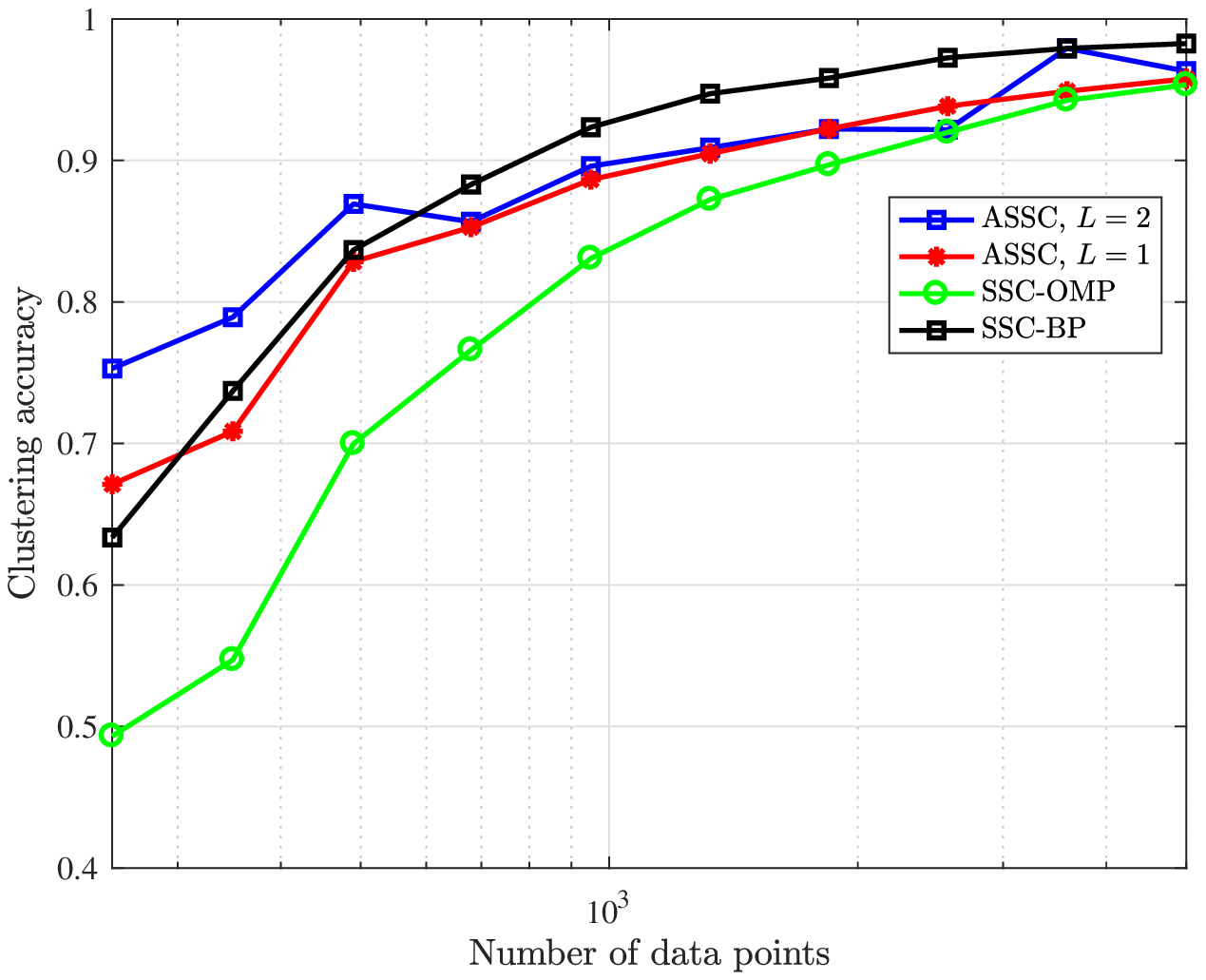}\quad\caption{\footnotesize  Clustering accuracy}
	\end{subfigure}
	\begin{subfigure}[]{.5\textwidth}
		\centering
		\includegraphics[width=1\textwidth]{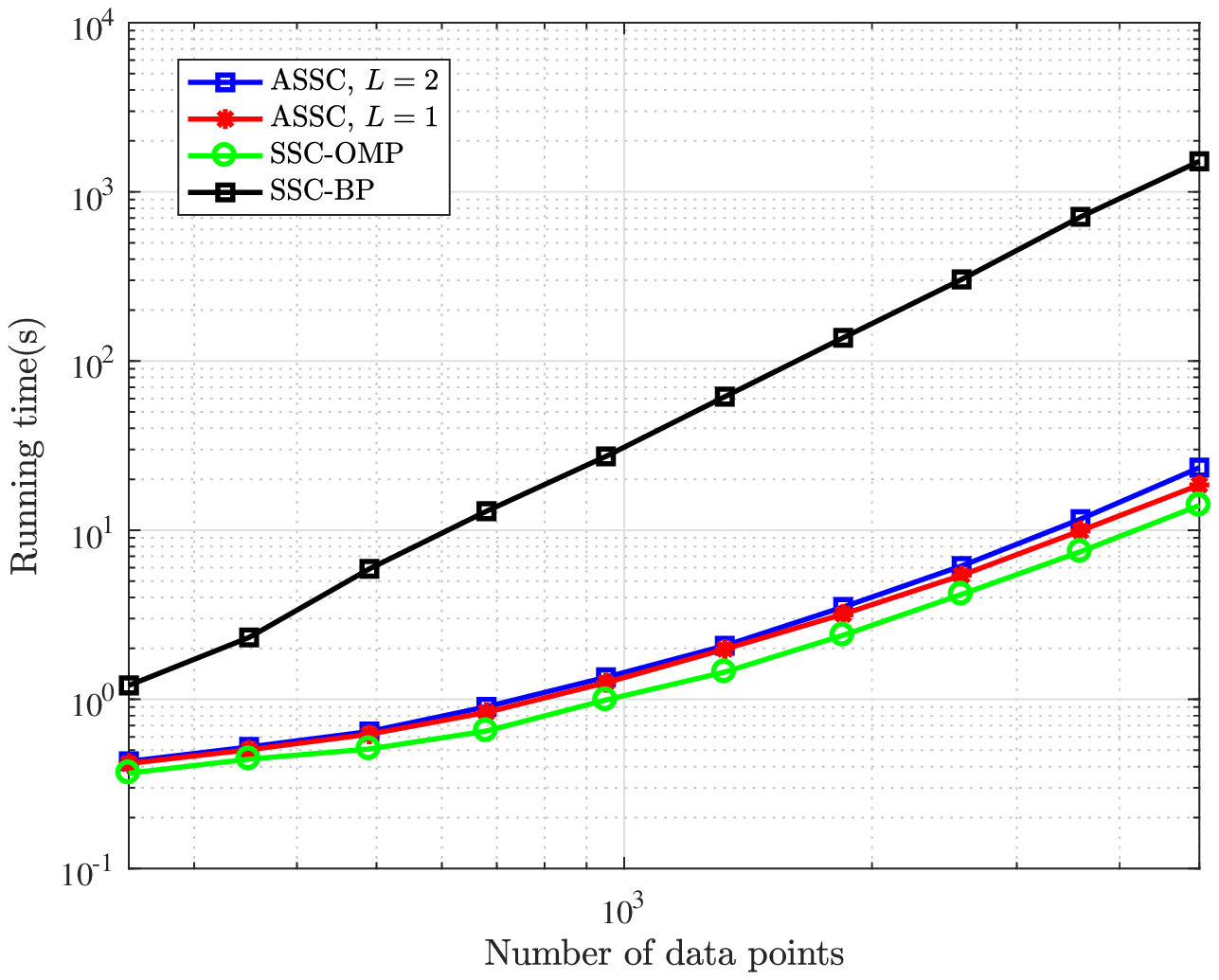}\quad\caption{\footnotesize Running time (sec)}
	\end{subfigure}
	\caption{\label{fig:ssc1} Performance comparison of ASSC, SSC-OMP \cite{dyer2013greedy,you2015sparse}, and SSC-BP \cite{elhamifar2009sparse,elhamifar2013sparse} {\color{black}{on synthetic data with no perturbation}}. The points are drawn from $5$ subspaces of dimension $6$ in ambient dimension $9$. Each subspace contains the same number of points and the overall number of points is varied from $250$ to $5000$.}   
\end{figure*}
\begin{figure*}[t]
\begin{subfigure}[]{0.5\textwidth}
		\centering
		\includegraphics[width=1\textwidth]{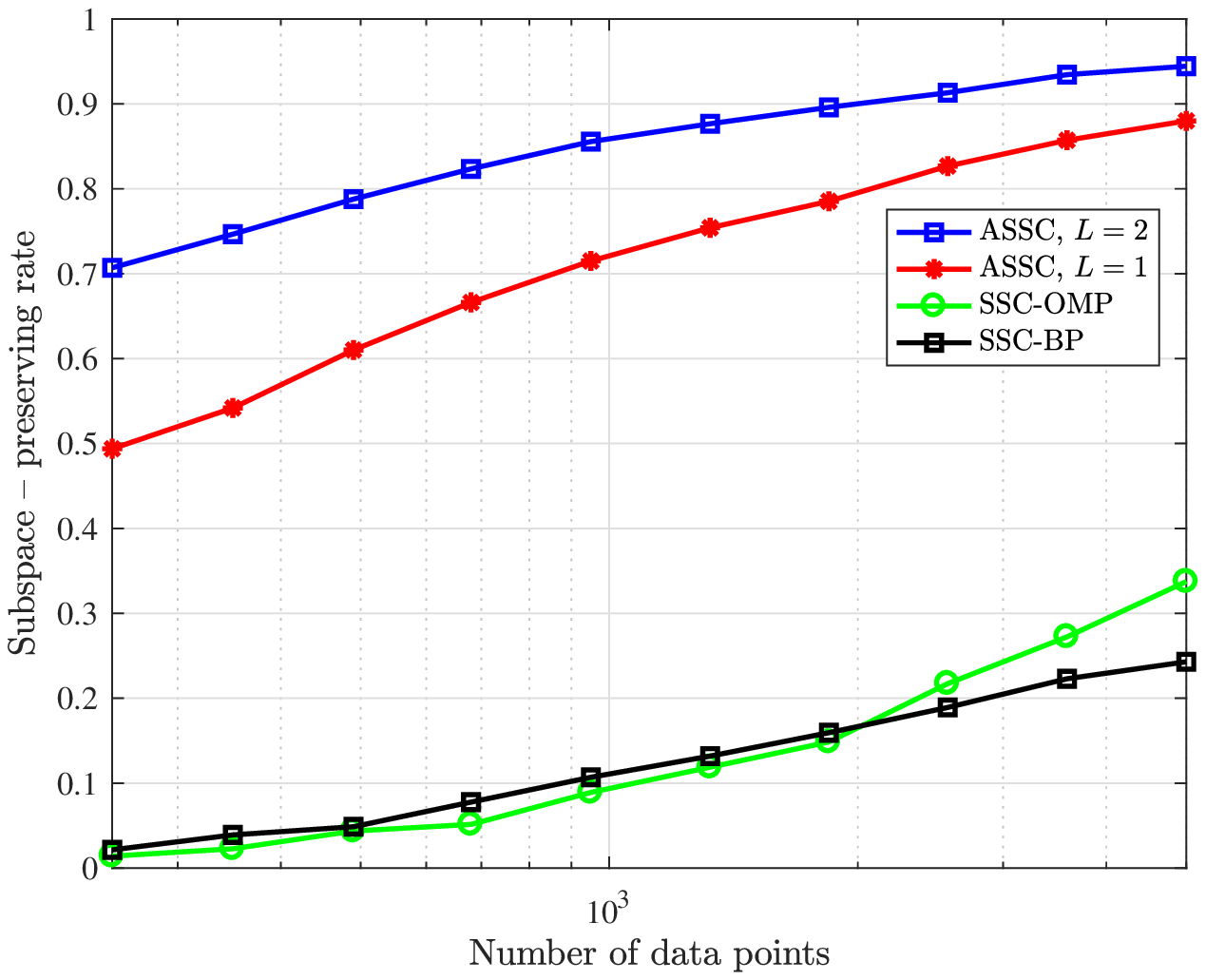}\quad\caption{\footnotesize Subspace preserving rate}
	\end{subfigure}
	\begin{subfigure}[]{0.5\textwidth}
		\centering
		\includegraphics[width=1\textwidth]{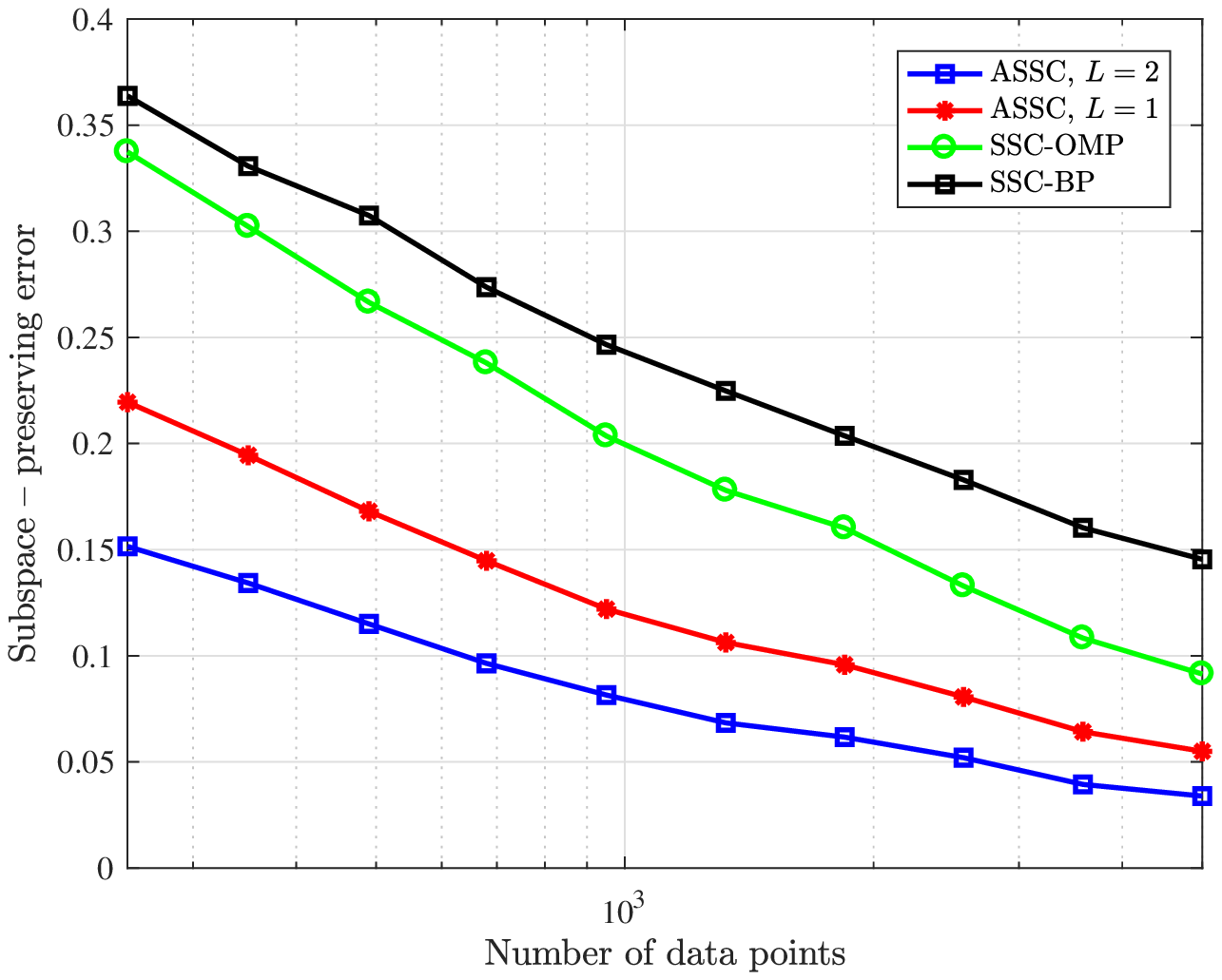}\quad\caption{\footnotesize  Subspace preserving error}
	\end{subfigure}
	\begin{subfigure}[]{.5\textwidth}
		\centering
		\includegraphics[width=1\textwidth]{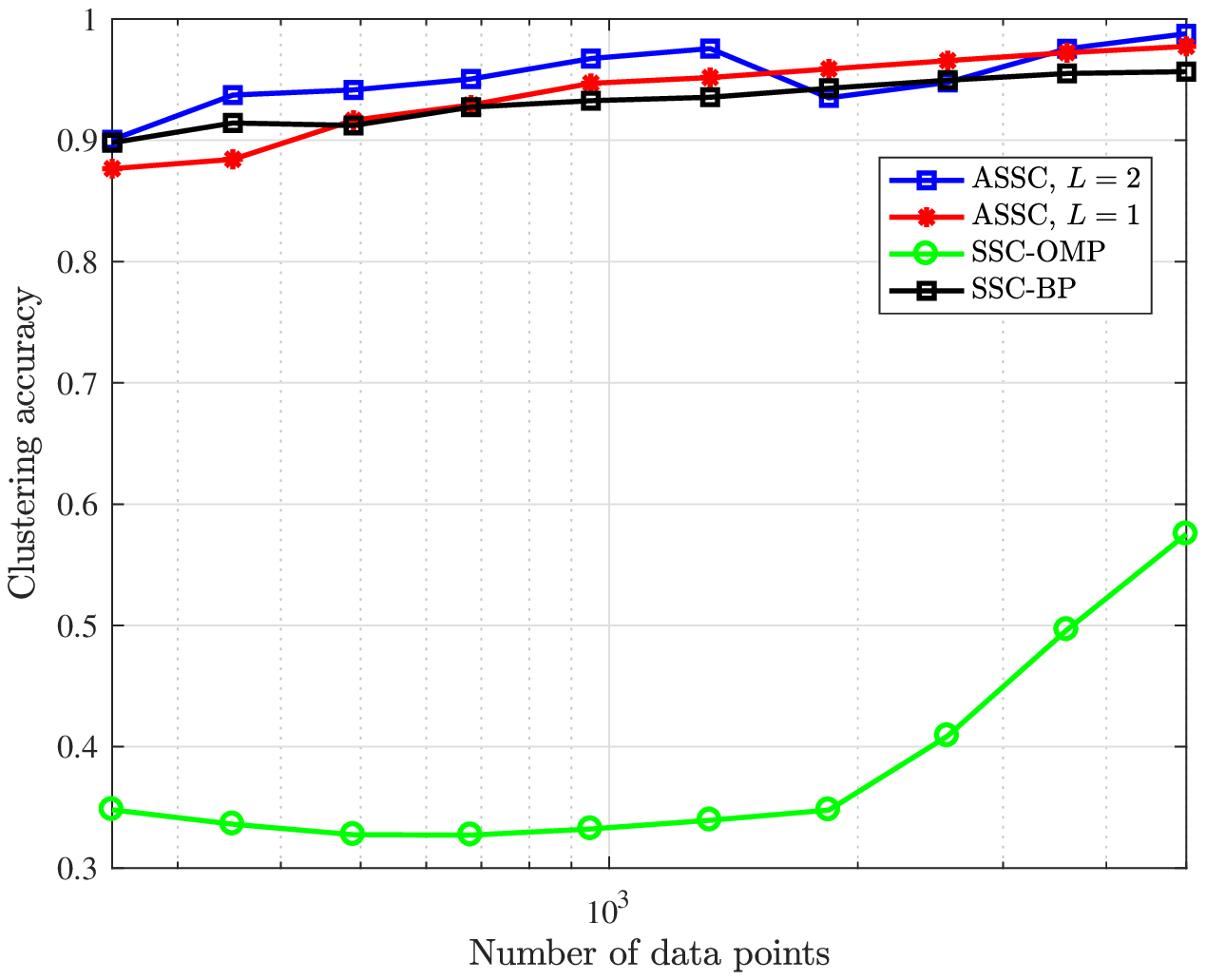}\quad\caption{\footnotesize Clustering accuracy}
	\end{subfigure}
	\begin{subfigure}[]{.5\textwidth}
		\centering
		\includegraphics[width=1\textwidth]{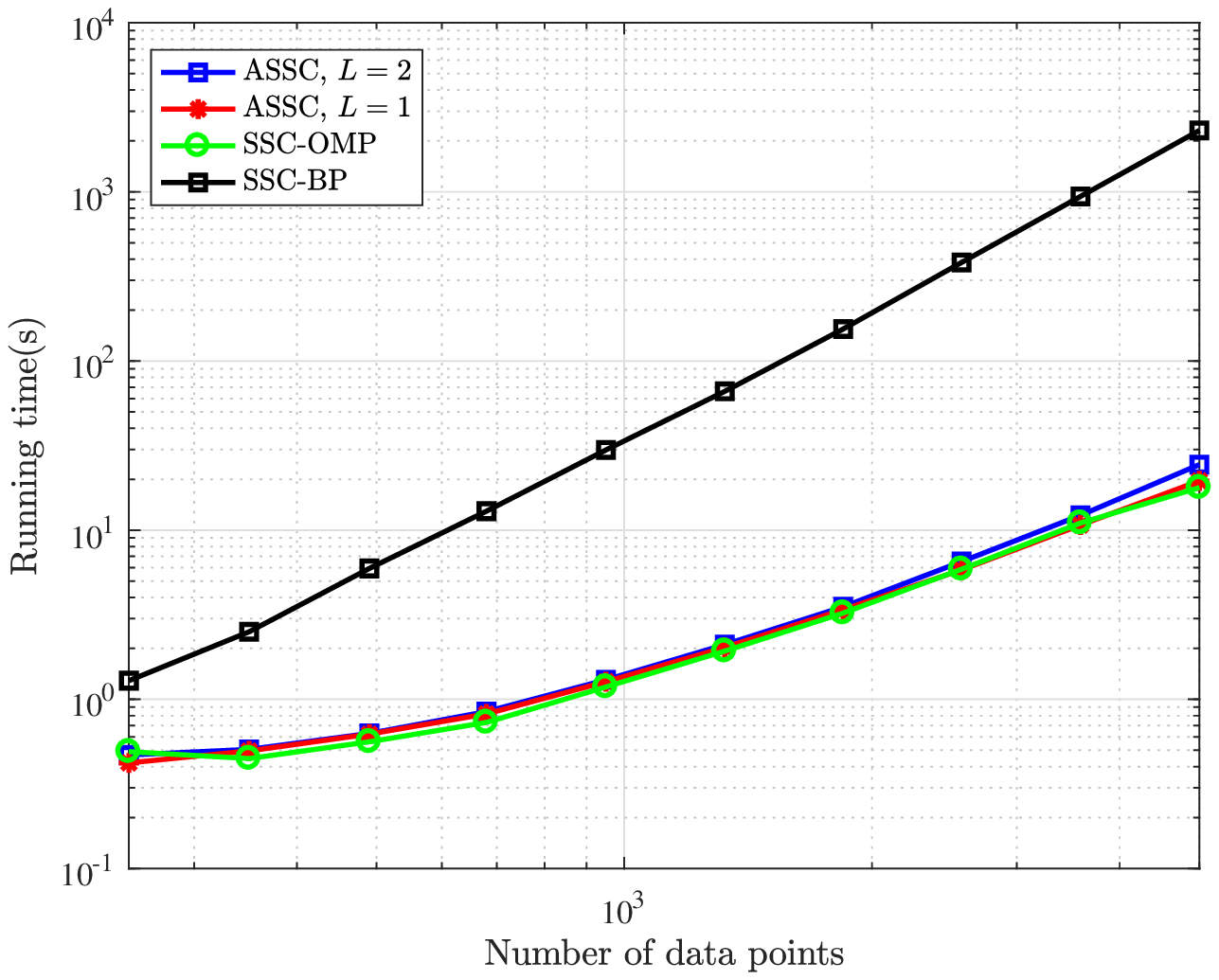}\quad\caption{\footnotesize Running time (sec)}
	\end{subfigure}
	\caption{\label{fig:ssc2} Performance comparison of ASSC, SSC-OMP \cite{dyer2013greedy,you2015sparse}, and SSC-BP \cite{elhamifar2009sparse,elhamifar2013sparse} {\color{black}{on synthetic data with perturbation terms $Q\sim \mathcal{U}(0,1)$}}. The  points are drawn from $5$ subspaces of dimension $6$ in ambient dimension $9$. Each subspace contains the same number of points and the overall number of points is varied from $250$ to $5000$.} 
\end{figure*}
Sparse subspace clustering (SSC), which received considerable attention in recent years, relies on sparse signal reconstruction techniques to organize high-dimensional data known to have low-dimensional representation \cite{elhamifar2009sparse}. In particular, in SSC problems we are given matrix $\A$ which collects data points $\a_i$ drawn from a union of low-dimensional subspaces, and are interested in partitioning the data according to their subspace membership. State-of-the-art SSC schemes such as SSC-OMP \cite{dyer2013greedy,you2015sparse} and SSC-BP \cite{elhamifar2009sparse,elhamifar2013sparse} typically consist of two steps. In the first step, one finds a similarity matrix $\mathbf{W}$ characterizing relative affinity of data points by forming a representation matrix $\C$.  Once $\mathbf{W} = |\C| + |\C|^\top$ is generated, the second step performs data segmentation by applying spectral clustering \cite{ng2001spectral} on $\mathbf{W}$. Most of the SSC methods rely on the so-called self-expressiveness property of data belonging to a union of subspaces which states that each point in a union of subspaces can be written as a linear combination of other points in the union \cite{elhamifar2009sparse}. 

In this section, we employ the proposed
AOLS algorithm to generate the subspace-preserving similarity matrix $\mathbf{W}$ and empirically compare the resulting SSC 
performance with that of SSC-OMP \cite{dyer2013greedy,you2015sparse} and SSC-MP \cite{elhamifar2009sparse,elhamifar2013sparse}.\footnote{We refer to our proposed scheme for the SSC problem as Accelerated SSC (ASSC).} 
For SSC-BP, two implementations based on ADMM and interior point methods are available by the authors of \cite{elhamifar2009sparse,elhamifar2013sparse}. In our simulation studies we use the ADMM implementation of SSC-BP in \cite{elhamifar2009sparse,elhamifar2013sparse} as it is faster than the interior point method implementation. Our scheme is tested for $L = 1$ and $L=2$. We consider the following two scenarios: (1) A random model where the subspaces are with high probability near-independent; and (2) The setting where we used hybrid dictionaries \cite{soussen2013joint} to generate similar data points across different subspaces which in turn implies the independence assumption no longer holds. In both scenarios, we randomly generate $n = 5$ subspaces, each of dimension $d = 6$, in an ambient space of dimension $D = 9$. Each subspace contains $N_i$ sample points  where we vary $N_i$ from $50$ to $1000$; therefore, the total number of data points, $N = \sum_{i=1}^n N_i$, is varied from $250$ to $5000$. The results are averaged over $20$ independent instances. For scenario (1), we generate data points by uniformly sampling from the unit sphere. For the second scenario, after sampling a data point we add a perturbation term $Q\mathbf{1}_D$ where $Q\sim \mathcal{U}(0,1)$. 

In addition to comparing the algorithms in terms of their clustering accuracy and running time, we use the following metrics defined in \cite{elhamifar2009sparse,elhamifar2013sparse} that quantify the subspace preserving property of the representation matrix $\C$ returned by each algorithm: 
 \begin{itemize}
 \item \textit{Subspace preserving rate:} The fraction of points whose representations are 
 subspace-preserving.
 \item \textit{Subspace preserving error:} {\color{black}{The fraction of $\ell_1$ norms of the representation 
 coefficients associated with points from other subspaces, i.e., 
 $\frac{1}{N}\sum_{j}{(\sum_{i\in O}{|\C_{ij}}|\slash \|\c_j\|_1)}$}} where $O$ represents the 
 set of data points from other subspaces.
 \end{itemize}

The results for the scenario (1) and (2) are illustrated in Fig. \oldref{fig:ssc1} and Fig. \oldref{fig:ssc2}, respectively. As can be seen in Fig. \oldref{fig:ssc1}, ASSC is nearly as fast as SSC-OMP and orders of magnitude faster than SSC-BP while ASSC achieves better subspace preserving rate, subspace preserving error, and clustering accuracy compared to competing schemes. In the second scenario, we observe that the performance of SSC-OMP is severely deteriorated while ASSC still outperforms both SSC-BP and SSC-OMP in terms of accuracy. Further, similar to the first scenario, running time of ASSC is similar to that of SSC-OMP while both methods are much faster that SSC-BP. As Fig. \oldref{fig:ssc1} and Fig. \oldref{fig:ssc2} suggest, the ASSC algorithm, especially with $L=2$, outperforms other schemes while essentially being as fast as the SSC-OMP method.
\section{Conclusions and Future Work}\label{sec:conc}
In this paper, we proposed the Accelerated Orthogonal Least-Squares (AOLS) algorithm, a novel 
scheme for sparse vector approximation. AOLS, unlike state-of-the art OLS-based schemes such 
as Multiple Orthogonal Least-Squares (MOLS) \cite{wang2017recovery}, relies on a set of 
expressions which provide computationally efficient recursive updates of the orthogonal projection 
operator and enable computation of the residual vector by employing only linear equations. 
Additionally, AOLS allows incorporating $L$ columns in each iteration to further reduce the complexity while achieving improved performance. In our theoretical analysis of AOLS, we showed that for coefficient matrices consisting of entries drawn from a Gaussian distribution, AOLS with high probability recovers $k$-sparse $m$-dimensional signals in at most $k$ iterations from ${\cal O}(k\log \frac{m}{k+L-1})$ noiseless random linear measurements. We extended this result to the scenario where the measurements are perturbed with $\ell_2$-bounded noise. Specifically, if the non-zero elements of an unknown vector are sufficiently large, ${\cal O}(k\log \frac{m}{k+L-1})$ random linear measurements is sufficient to guarantee recovery with high probability. This asymptotic bound on the required number of measurements is lower than those of the existing OLS-based, OMP-based, and convex relaxation schemes. Our simulation results verify that ${\cal O}\left(k\log \frac{m}{k+L-1}\right)$ measurement is indeed sufficient for sparse reconstruction that is exact with probability arbitrarily close to one. Simulation studies demonstrate that AOLS outperforms all of the current state-of-the-art methods in terms of both accuracy and running time. Furthermore, we considered an application to sparse subspace clustering where we employed AOLS to facilitate efficient clustering of high-dimensional data points lying on the union of low-dimensional subspaces, showing superior performance compared to state-of-the-art 
OMP-based and BP-based methods \cite{you2015sparse,dyer2013greedy,elhamifar2009sparse,elhamifar2013sparse}.

As part of future work, it would be valuable to further extend the analysis carried out in Section \oldref{sec:Gua} to study performance of AOLS for hybrid dictionaries \cite{soussen2013joint}. It is also of interest to analytically characterize performance of the AOLS-based sparse subspace clustering scheme.
\appendix
\section*{Appendices}
\section{Proof of Lemma \oldref{lem:23}} \label{pf:lem2}

The lemma aims to characterize the length of the projection of a random vector onto a 
low dimensional subspace. In the following argument we show that the distribution of 
the length of the projected vector is invariant to rotation which in turn enables us to find 
the projection in a straightforward manner.

Recall that $\P_k$ is an orthogonal projection operator for a $k$-dimensional subspace ${\cal L}_k$ spanned 
by the columns of $\A_k$. Let ${\cal B}=\{{\b_1},\dots,{\b_k}\}$ denote an orthonormal basis for ${\cal L}_k$. 
There exist a rotation operator $\cal R$ such that ${\cal R}\left({\cal B}\right)=\{{\mathbf{e}_1},\dots,{\mathbf{e}_k}\}$, 
where $\mathbf{e}_i$ is the $i\ts{th}$ standard unit vector.
Let $\u \sim {\cal N}(0,1/n)$. Since a multivariate Gaussian 
distribution is spherically symmetric \cite{bochner2016lectures}, distribution of $\u$ remains unchanged under rotation, i.e., 
${\cal R}\left(\u\right)\sim{\cal N}(0,1/n)$. 
Therefore, it holds that 
$\E\norm{{\cal R}\left(\u\right)}_2=\E\norm{\u}_2$. In addition, since after rotation $\{{\mathbf{e}_1},\dots,{\mathbf{e}_k}\}$ is a basis for the rotation of ${\cal L}_k$,  $\norm{\P_k\u}_2$ has the same distribution as the length of a vector consisting of the first $k$ components of 
${\cal R}\left(\u\right)$. It then follows from the i.i.d. assumption and linearity of 
expectation that $\E\norm{\P_k\u}_2^2=\frac{k}{n}\E\norm{\u}_2^2=\frac{k}{n}$.

We now prove the statement in the second part of the lemma. Let $\u_k^{\cal R}$ be the vector collecting the first $k$ coordinates of ${\cal R}(\u)$. The above argument implies $\norm{\P_k\u}_2^2$ has the same distribution as $\norm{\u_k^{\cal R}}_2^2$. In addition, $n\norm{\u_k^{\cal R}}_2^2$ is distributed as $\chi^2_k$ because of the spherical symmetry property of $\u$. Let $\lambda >0$; we will specify the value of $\lambda$ shortly. Now,
\begin{equation}
\begin{aligned}
\Pr\{\norm{\P_k\u}_2^2\leq(1-\epsilon)\frac{k}{n}\}&=\Pr\{n\norm{\u_k^{\cal R}}_2^2\leq(1-\epsilon)k\}\\
&=\Pr\{-\frac{\lambda}{2}n\norm{\u_k^{\cal R}}_2^2\geq-\frac{\lambda k(1-\epsilon)}{2}\}\\
&=\Pr\{e^{-\frac{\lambda}{2}n\norm{\u_k^{\cal R}}_2^2}\geq e^{-\frac{\lambda k(1-\epsilon)}{2}}\}\\
&\stackrel{(a)}{\leq} e^{\frac{\lambda k(1-\epsilon)}{2}}\E\{e^{-\frac{\lambda}{2}n\norm{\u_k^{\cal R}}_2^2}\}\\
&\stackrel{(b)}{=}e^{\frac{\lambda k(1-\epsilon)}{2}}(1+\lambda)^\frac{-k}{2}
\end{aligned}
\end{equation}
where (a) follows from the Markov inequality and (b) is due to the definition of the Moment Generating Function 
(MGF) for $\chi^2_k$-distribution. Now, let $\lambda=\frac{\epsilon}{1-\epsilon}$. It follows that
\begin{equation}\label{concet1}
\Pr\{\norm{\P_k\u}_2^2\leq(1-\epsilon)\frac{k}{n}\}\leq e^{\frac{\lambda k(1-\epsilon)}{2}}(1-\epsilon)^\frac{k}{2}=e^{\frac{k}{2}(\epsilon+\log(1-\epsilon))}\leq e^{\frac{-k\epsilon^2}{4}}
\end{equation}
where in the last inequality we used the fact that $\log(1-\epsilon)\leq -\epsilon-\frac{\epsilon^2}{2}$. 
Following the same line of argument, one can show that
\begin{equation}\label{concet2}
\Pr\{\norm{\P_k\u}_2^2\geq(1+\epsilon)\frac{k}{n}\}\leq e^{-k(\frac{\epsilon^2}{4}-\frac{\epsilon^3}{6})}.
\end{equation}
The combination of \ref{concet1} and \ref{concet2} using Boole's inequality leads to the stated result.
\section{Proof of Theorem \oldref{thm:nois}} \label{pf:thm2}
Here we follow the outline of the proof of Theorem \oldref{thm:1}. Note that, 
in the presence of noise, $\bar{\A}^\top \r_i$ in \ref{eq:thm11} has at most $k$ nonzero entries. After a 
straightforward modification of \ref{eq:thm1}, we obtain
\begin{equation}
\rho(\r_i)\leq \frac{\sqrt{k}}{c_1(\epsilon)}|\mathcal{P}(\widetilde{\A}^\top \r_i)_{m-k-L+1}|.
\end{equation}
The most important difference between the noisy and noiseless scenarios is that $\r_i$ in the latter does not belong 
to the range of $\bar{\A}$; therefore, further 
restrictions are needed to ensure that $\{\widetilde{\r}_i\}_{i=0}^{k-1}$ remains bounded. To this end, we investigate
lower bounds on $\|\bar{\A}^\top \r_i\|_2$ and upper bounds on $\|\widetilde{\r}_i\|_2$. 
Recall that in the $i\ts{th}$ iteration
\begin{equation}\label{eq:resdef}
\r_i=\P_i^\bot\y=\P_i^\bot\left(\bar{\A}\bar{\x}+\e\right),
\end{equation}
where $\bar{\x} \in \R^k$ is a subvector of $\x$ that collects nonzero components of $\x$. We can write $\e$ equivalently as 
\begin{equation} \label{eq:noisep}
\e=\bar{\A}\w+\e^\bot,
\end{equation}
where $\e^\bot=\P_k^\bot\e$ is the projection of $\e$ onto the orthogonal complement of the subspace 
spanned by the 
columns of $\A$ corresponding to nonzero entries of $\x$, and $\w=\bar{\A}^\dagger\e$. Substituting \ref{eq:noisep} 
into \ref{eq:resdef} and noting that $\P_i^\bot\a=0$ if $\a$ is selected in previous iterations as well as observing that 
${\cal L}_i \subset {\cal L}_k$, we obtain
\begin{equation}\label{eq:resort}
\r_i=\e^\bot+\P_i^\bot\bar{\A}_{i^c}\c_{i^c},
\end{equation}
where $\c=\bar{\x}+\w$ and subscript $i^c$ denotes the set of correct columns that have not yet been selected. 
Evidently, \ref{eq:resort} demonstrates that $\r_i$ can be written as a sum of orthogonal terms. Therefore,
\begin{equation}\label{eq:resnorm}
\|\r_i\|_2^2=\|\e^\bot\|_2^2+\|\P_i^\bot\bar{\A}_{i^c}\c_{i^c}\|_2^2.
\end{equation}
Applying \ref{eq:resort} yields
\begin{equation}\label{eq:lb}
\begin{aligned}
\|\bar{\A}^\top\r_i\|_2&=\|\bar{\A}^\top\left(\e^\bot+\P_i^\bot\bar{\A}_{i^c}\c_{i^c}\right)\|_2\\
&\stackrel{(a)}{=}\|\bar{\A}^\top\e^\bot+\bar{\A}_{i^c}^\top\P_i^\bot\bar{\A}_{i^c}\c_{i^c}\|_2\\
&\stackrel{(b)}{=}\|\bar{\A}_{i^c}^\top\P_i^\bot\bar{\A}_{i^c}\c_{i^c}\|_2\\
&\stackrel{(c)}{\geq} \sigma_{\min}^2(\bar{\A})\|\c_{i^c}\|_2,
\end{aligned}
\end{equation}
where ($a$) holds because $\P_i^\bot$ projects onto the orthogonal complement of the space spanned by the columns 
of $\bar{\A}_i$, ($b$) follows from the fact that columns of $\bar{\A}$ and $\e^\bot$ lie in orthogonal subspaces, and ($c$) 
follows from Lemma \oldref{lem:nois} and the fact that $\P_i^\bot$ is a projection matrix.

We now bound the norm of $\widetilde{\r}_i$. Substitute \ref{eq:resnorm} and \ref{eq:lb} in the definition of $\widetilde{\r}_i$ 
to arrive at
\begin{equation}\label{eq:normrt}
\begin{aligned}
\|\widetilde{\r}_i\|_2 &\leq \frac{\left[\|\e^\bot\|_2^2+\|\P_i^\bot\bar{\A}_{i^c}\c_{i^c}\|_2^2\right]^{\frac{1}{2}}}{\sigma_{\min}^2(\bar{\A})\|\c_{i^c}\|_2} \\
&\stackrel{(a)}{\leq} \frac{\left[\|\e^\bot\|_2^2+\sigma_{\max}^2(\bar{\A})\|\c_{i^c}\|_2^2\right]^{\frac{1}{2}}}{\sigma_{\min}^2(\bar{\A})\|\c_{i^c}\|_2} \\
&= \frac{\left[\|\e^\bot\|_2^2\slash \|\c_{i^c}\|_2^2+\sigma_{\max}^2(\bar{\A})\right]^{\frac{1}{2}}}{\sigma_{\min}^2(\bar{\A})}
\end{aligned}
\end{equation}
where ($a$) follows from Lemma \oldref{lem:nois} and the fact that $\P_i^\bot$ is a projection matrix. 
In addition, 
\begin{equation} \label{eq:normvperp}
\begin{aligned}
\|\e^\bot\|_2=\|\P_k^\bot\e\|_2\leq\|\e\|_2\leq \epsilon_{\e}.
\end{aligned}
\end{equation}
Defining $\x_{\min}=\min_{j}{|\bar{\x}_j|}$ and $\c_{\min}=\min_{j}{|\c_j|}$, it is straightforward to see that
\begin{equation}\label{eq:lbc}
\begin{aligned}
\c_{\min}\geq\x_{\min}-\|\w\|_2.
\end{aligned}
\end{equation}
Moreover, we impose $\x_{\min}\geq(1+\delta)\|\w\|_2$. Therefore,
\begin{equation} \label{eq:normc}
\begin{aligned}
\|\c_{i^c}\|_2^2 &\geq (k-i) \c_{\min}^2 \\
&\geq (k-i) \left(\x_{\min}-\|\w\|_2\right)^2 \\
&= (k-i) (\x_{\min}-\|\bar{\A}^{\dagger}\e\|_2)^2 \\
&\geq (k-i) (\x_{\min}-\sigma_{\max}(\bar{\A}^{\dagger})\|\e\|_2)^2 \\
&=(k-i) (\x_{\min}-\sigma_{\min}(\bar{\A})\epsilon_{\e})^2. 
\end{aligned}
\end{equation}
Combining \ref{eq:normrt}, \ref{eq:normvperp}, and \ref{eq:normc} implies that
\begin{equation}
\begin{aligned}
\|\widetilde{\r}_i\|_2 &\leq \frac{\left[\frac{\epsilon_{\e}^2}{(k-i) (\x_{\min}-\sigma_{\min}(\bar{\A})\epsilon_{\e})^2}+\sigma_{\max}^2(\bar{\A})\right]^{\frac{1}{2}}}{\sigma_{\min}^2(\bar{\A})}\\
&\leq \frac{\left[\frac{\epsilon_{\e}^2}{(k-i) (\x_{\min}-(1+\delta)\epsilon_{\e})^2}+(1+\delta)^2\right]^{\frac{1}{2}}}{(1-\delta)^2}
\end{aligned}
\end{equation}
with probability exceeding $p_2$. Thus, imposing the constraint 
\begin{equation} \label{eq:xminlb}
\x_{\min}\geq (1+\delta+t)\epsilon_{\e}
\end{equation}
where $t>0$\footnote{This is consistent with our previous condition $\x_{\min}\geq(1+\delta)\|\w\|_2$.} establishes
\begin{equation} \label{eq:rtilub}
\begin{aligned}
\|\widetilde{\r}_i\|_2 \leq \frac{\left[\frac{1}{(k-i) t^2}+(1+\delta)^2\right]^{\frac{1}{2}}}{(1-\delta)^2}.
\end{aligned}
\end{equation}
By following the steps of the proof of Theorem \oldref{thm:1} and exploiting independence of the columns 
of $\widetilde{\A}$, we arrive at
\begin{equation} \label{eq:noisesucc}
\Pr\{\Sigma\}\geq p_1 p_2 \Pr\{\max_{0\leq i <k}{\left|\widetilde{\a}_{o_1}^\top \widetilde{\r}_i\right|}<\frac{c_1(\epsilon)}{\sqrt{k}}\}^{m-k-L+1}.
\end{equation}
Recall that $\{\widetilde{\r}_i\}_{i=0}^{k-1}$ are statistically independent of $\widetilde{\A}$ and that with probability higher 
than $p_2$ they are bounded. By using Boole's
for the random variable 
$X_{i}=\widetilde{\a}_{o_1}^\top \widetilde{\r}_i$ we obtain
\begin{equation}\label{eq:noisep3}
\Pr\{\max_{0\leq i <k}{\left|X_{i}\right|< \frac{c_1(\epsilon)}{\sqrt{k}}}\}\geq 1-\sum_{i=0}^{k-1} 
e^{-\frac{n c_1(\epsilon)^2(1-\delta)^4}{k\left[\frac{1}{(k-i) t^2}+(1+\delta)^2\right]}}.
\end{equation}
Let us denote
\begin{equation}
p_3=\left(1-\sum_{i=0}^{k-1} e^{-\frac{n c_1(\epsilon)^2(1-\delta)^4}{k\left[\frac{1}{(k-i) t^2}+(1+\delta)^2\right]}}\right)^{m-k-L+1}.
\end{equation}
Then from \ref{eq:noisesucc} and \ref{eq:noisep3} follows that $\Pr\{\Sigma\} \geq p_1p_2p_3$, which 
completes the proof.

\textit{Remark 4:} Note that in the absence of noise the first term in the numerator of \ref{eq:rtilub} vanishes, leading to 
$\|\widetilde{\r}_i\|_2 \leq \frac{1}{1-\delta}+\frac{2\delta}{(1-\delta)^2}$. A comparison with the proof of Theorem \oldref{thm:1}
suggests that the term $\frac{2\delta}{(1-\delta)^2}$ is a modification which stems from the presence of noise.
\clearpage
\bibliographystyle{ieeetr}\footnotesize
\bibliography{refs}
\end{document}